\documentclass{article}
\usepackage{geometry}
\usepackage{amssymb}
\usepackage{amsmath}
\usepackage{amsfonts}
\usepackage{xcolor}
\usepackage[colorlinks=true, linkcolor=blue, citecolor=blue]{hyperref}
\usepackage[nice]{nicefrac}
\usepackage{amsthm, mathrsfs}
\usepackage{natbib}
\usepackage{url}
\usepackage{enumitem}
\usepackage{mathtools}
\usepackage{graphicx}
\newtheorem{theorem}{Theorem}
\newtheorem{definition}{Definition}
\newtheorem{proposition}{Proposition}
\newtheorem{corollary}{Corollary}
\newtheorem{lemma}{Lemma}

\newcommand{\numberthis}{\addtocounter{equation}{1}\tag{\theequation}}

\usepackage{color-edits}
\addauthor{sr}{red}
\addauthor{zj}{green}

\allowdisplaybreaks

\usepackage{makecell}

\usepackage{prettyref}
\newcommand{\pref}[1]{\prettyref{#1}}

\newcommand{\savehyperref}[2]{\texorpdfstring{\hyperref[#1]{#2}}{#2}}
\newrefformat{eq}{\savehyperref{#1}{Eq.~\textup{(\ref*{#1})}}}
\newrefformat{eqn}{\savehyperref{#1}{Equation~\ref*{#1}}}
\newrefformat{lem}{\savehyperref{#1}{Lemma~\ref*{#1}}}
\newrefformat{lemma}{\savehyperref{#1}{Lemma~\ref*{#1}}}
\newrefformat{def}{\savehyperref{#1}{Definition~\ref*{#1}}}
\newrefformat{line}{\savehyperref{#1}{Line~\ref*{#1}}}
\newrefformat{thm}{\savehyperref{#1}{Theorem~\ref*{#1}}}
\newrefformat{corr}{\savehyperref{#1}{Corollary~\ref*{#1}}}
\newrefformat{cor}{\savehyperref{#1}{Corollary~\ref*{#1}}}
\newrefformat{sec}{\savehyperref{#1}{Section~\ref*{#1}}}
\newrefformat{subsec}{\savehyperref{#1}{Section~\ref*{#1}}}
\newrefformat{app}{\savehyperref{#1}{Appendix~\ref*{#1}}}
\newrefformat{ass}{\savehyperref{#1}{Assumption~\ref*{#1}}}
\newrefformat{ex}{\savehyperref{#1}{Example~\ref*{#1}}}
\newrefformat{fig}{\savehyperref{#1}{Figure~\ref*{#1}}}
\newrefformat{alg}{\savehyperref{#1}{Algorithm~\ref*{#1}}}
\newrefformat{rem}{\savehyperref{#1}{Remark~\ref*{#1}}}
\newrefformat{conj}{\savehyperref{#1}{Conjecture~\ref*{#1}}}
\newrefformat{prop}{\savehyperref{#1}{Proposition~\ref*{#1}}}
\newrefformat{proto}{\savehyperref{#1}{Protocol~\ref*{#1}}}
\newrefformat{prob}{\savehyperref{#1}{Problem~\ref*{#1}}}
\newrefformat{claim}{\savehyperref{#1}{Claim~\ref*{#1}}}
\newrefformat{que}{\savehyperref{#1}{Question~\ref*{#1}}}
\newrefformat{op}{\savehyperref{#1}{Open Problem~\ref*{#1}}}
\newrefformat{fn}{\savehyperref{#1}{Footnote~\ref*{#1}}}
\newrefformat{tab}{\savehyperref{#1}{Table~\ref*{#1}}}
\newrefformat{fig}{\savehyperref{#1}{Figure~\ref*{#1}}}
\newrefformat{proc}{\savehyperref{#1}{Procedure~\ref*{#1}}}
\newrefformat{property}{\savehyperref{#1}{Property~\ref*{#1}}}

\newcommand{\EE}{\mathbb{E}}
\newcommand{\RR}{\mathbb{R}}

\newcommand{\calX}{\mathcal{X}}
\newcommand{\calV}{\mathcal{V}}
\newcommand{\calF}{\mathcal{F}}
\newcommand{\calK}{\mathcal{K}}

\newcommand{\calN}{\mathcal{N}}

\newcommand{\calM}{\mathcal{M}}

\newcommand{\calR}{\mathcal{R}}

\newcommand{\tO}{\tilde{\mathcal{O}}}
\newcommand{\bs}{\mathbf{s}}
\newcommand{\by}{\mathbf{y}}

\newcommand{\bx}{\mathbf{x}}
\newcommand{\bp}{\mathbf{p}}

\newcommand{\bv}{\mathbf{v}}

\newcommand{\hy}{\widehat{y}}
\newcommand{\tx}{\tilde{x}}

\newcommand{\btx}{\tilde{\bx}}

\newcommand{\be}{\mathbf{e}}

\newcommand{\unif}{\mathrm{Unif}}
\newcommand{\argmin}{\mathop{\arg\min}} 

\newcommand{\bepsilon}{{\pmb{\epsilon}}}

\renewcommand{\epsilon}{\varepsilon}
\newcommand{\simiid}{\stackrel{\mathrm{i.i.d.}}{\sim}}

\newcommand{\frakc}{\mathfrak{c}}
\newcommand{\barf}{\bar{f}}
\newcommand{\bars}{\bar{s}}
\newcommand{\barbs}{\bar{\bs}}
\newcommand{\barF}{\bar{\calF}}
\newcommand{\barV}{\bar{\calV}}
\newcommand{\barv}{\bar{v}}
\newcommand{\barbv}{\bar{\bv}}

\newcommand{\btepsilon}{\tilde{\bepsilon}}
\newcommand{\tepsilon}{\tilde{\epsilon}}
\newcommand{\hepsilon}{\widehat{\epsilon}}

\newcommand{\bmu}{\pmb{\mu}}
\newcommand{\tmu}{\tilde{\mu}}
\newcommand{\btmu}{\tilde{\bmu}}
\newcommand{\sign}{\mathrm{sign}}
\newcommand{\fix}{{\sf{fix}}}
\newcommand{\tcalF}{\widetilde{\calF}}
\newcommand{\calT}{\mathcal{T}}
\newcommand{\tcalX}{\tilde{\mathcal{X}}}

\newcommand{\vc}{\mathsf{vc}}
\newcommand{\ldim}{\mathsf{ldim}}
\newcommand{\fat}{\mathrm{sfat}}

\newcommand{\reals}{\mathbb{R}}
\newcommand{\En}{\mathbb{E}}
\newcommand{\veps}{\varepsilon}

\newcommand{\dseq}{d^{\sf{seq}}}
\newcommand{\dnon}{d^{\sf{}}}
\newcommand{\calNnon}{\calN^{\sf{}}}
\newcommand{\calNseq}{\calN^{\sf{seq}}}
\newcommand{\calVseq}{\calV^{\sf{seq}}}
\newcommand{\tOmega}{\tilde{\Omega}}

\def\ddefloop#1{\ifx\ddefloop#1\else\ddef{#1}\expandafter\ddefloop\fi}
\def\ddef#1{\expandafter\def\csname bb#1\endcsname{\ensuremath{\mathbb{#1}}}}
\ddefloop ABCDEFGHIJKLMNOPQRSTUVWXYZ\ddefloop
\def\ddefloop#1{\ifx\ddefloop#1\else\ddef{#1}\expandafter\ddefloop\fi}
\def\ddef#1{\expandafter\def\csname b#1\endcsname{\ensuremath{\mathbf{#1}}}}
\ddefloop ABCDEFGHIJKLMNOPQRSTUVWXYZ\ddefloop
\def\ddef#1{\expandafter\def\csname sf#1\endcsname{\ensuremath{\mathsf{#1}}}}
\ddefloop ABCDEFGHIJKLMNOPQRSTUVWXYZ\ddefloop
\def\ddef#1{\expandafter\def\csname c#1\endcsname{\ensuremath{\mathcal{#1}}}}
\ddefloop ABCDEFGHIJKLMNOPQRSTUVWXYZ\ddefloop
\def\ddef#1{\expandafter\def\csname h#1\endcsname{\ensuremath{\widehat{#1}}}}
\ddefloop ABCDEFGHIJKLMNOPQRSTUVWXYZ\ddefloop
\def\ddef#1{\expandafter\def\csname hc#1\endcsname{\ensuremath{\widehat{\mathcal{#1}}}}}
\ddefloop ABCDEFGHIJKLMNOPQRSTUVWXYZ\ddefloop
\def\ddef#1{\expandafter\def\csname t#1\endcsname{\ensuremath{\widetilde{#1}}}}
\ddefloop ABCDEFGHIJKLMNOPQRSTUVWXYZ\ddefloop
\def\ddef#1{\expandafter\def\csname tc#1\endcsname{\ensuremath{\widetilde{\mathcal{#1}}}}}
\ddefloop ABCDEFGHIJKLMNOPQRSTUVWXYZ\ddefloop
\def\ddefloop#1{\ifx\ddefloop#1\else\ddef{#1}\expandafter\ddefloop\fi}
\def\ddef#1{\expandafter\def\csname scr#1\endcsname{\ensuremath{\mathscr{#1}}}}
\ddefloop ABCDEFGHIJKLMNOPQRSTUVWXYZ\ddefloop

\title{A Gapped Scale-Sensitive Dimension and Lower Bounds for Offset Rademacher Complexity}
	\author{
		Zeyu Jia
		\\
		\normalsize
        {\texttt{zyjia@mit.edu}}
		\and
		Yury Polyanskiy
		\\
		\normalsize
        {\texttt{yp@mit.edu}}
		\and
		Alexander Rakhlin
		\\
		\normalsize
        {\texttt{rakhlin@mit.edu}}
		\and
	}

\begin{document}
\maketitle

\begin{abstract}
    We study gapped scale-sensitive dimensions of a function class in both sequential and non-sequential settings. We demonstrate that covering numbers for any uniformly bounded class are controlled above by these gapped dimensions, generalizing the results of \cite{anthony2000function,alon1997scale}. Moreover, we show that the gapped dimensions lead to lower bounds on offset Rademacher averages, thereby strengthening existing approaches for proving lower bounds on rates of convergence in statistical and online learning.
\end{abstract}

\section{Introduction} 
 
The celebrated Vapnik-Chervonenkis dimension $\vc(\cF)$ of a binary-valued function class $\cF$ and the scale-sensitive dimension $\vc(\cF,\alpha)$ of a real-valued function class $\cF$ are central notions in the study of empirical processes and convergence of statistical learning methods \cite{VapChe71,bartlett1994fat,kearns1994efficient}. Sequential analogues of these notions---the Littlestone dimension $\ldim(\cF)$ and the sequential scale-sensitive dimension $\fat(\cF,\alpha)$---have been shown to play an analogously central role in the study of uniform martingale laws and online prediction \cite{littlestone1988learning, ben2009agnostic, rakhlin2010online}. 

In this paper, we study ``gapped'' versions of $\vc(\cF,\alpha)$ and $\fat(\cF,\alpha)$. The modification yields a dimension that is no larger than the original one, yet can still be shown to control covering numbers in both sequential and non-sequential cases. More importantly, the new notion gives us a more precise control on the functions involved in ``shattering'' and thus yields non-vacuous lower bounds for offset Rademacher complexities for \textit{any} uniformly bounded class---both in the classical and sequential cases---and, as a consequence, tighter lower bounds for online prediction problems, such as online regression or transductive learning. Our definition in the non-sequential case can also be seen as a modification of the Natarajan dimension \cite{natarajan1988two,natarajan1989learning}, and was, in fact, introduced in \cite{anthony2000function}.

We first motivate the development in this paper on the simpler case of non-sequential data. We start by recalling the definition of the Vapnik-Chervonenkis dimension and its scale-sensitive version. Given a class $\cF\subseteq \{f:\cX\to\{-1,1\}\}$ of binary-valued functions on some set $\cX$, consider the projection $\cF|_{x_1,\ldots,x_d}=\left\{(f(x_1),\ldots,f(x_d)):f\in\cF\right\}$ onto $d$ elements $x_1,\ldots,x_d\in\cX$. The VC-dimension $\vc(\cF)$ is the largest $d$ such that there exist  $\{x_1,\ldots,x_d\}$ with $\cF|_{x_1,\ldots,x_d}=\{-1,1\}^d$. For a real-valued class $\cF\subseteq \{f:\cX\to[-1, 1]\}$ and a scale $\alpha\geq 0$, the scale-sensitive dimension $\vc(\cF,\alpha)$ is defined to be the largest $d$ such that there exist $x_1,\ldots,x_d \in \cX$ and $\bs\in[-1, 1]^d$ with the following property: for any $\bepsilon\in\{-1,1\}^d$ there exists $f^{\bepsilon}\in\cF$ with $\veps_i(f^{\bepsilon}(x_i)-s_i)\geq \alpha/2$ for all $i\in \{1,\ldots,d\}$. We say that $\cF$ \textit{shatters} $x_1,\ldots,x_n$ at scale $\alpha$ if the aforementioned property holds. Shattering can also be visualized as a property that $\cF|_{x_1,\ldots,x_d}$ ``contains'' a cube $\bs+(\alpha/2)\{-1,+1\}^d$ at scale $\alpha$ with a center at $\bs$, in the sense that for each direction $\bepsilon$, there is a function $f^{\bepsilon}\in\cF$ whose projection onto the data lies in the quadrant outside the vertex $\bs+(\alpha/2)\bepsilon$. 

Note that the definition of shattering does not tell us whether $f^{\bepsilon}$ is close to the corresponding vertex, i.e. 
\begin{align}
\label{eq:closeness_nonsequential}
    f^{\veps}(x_i)\approx s_i+\veps_i \alpha/2
\end{align}
for every $i\in\{1,\ldots,d\}$. We can see that such a requirement (in the non-sequential case described here) can be satisfied under the assumption of convexity of $\cF$. 
Unfortunately, for the sequential case described in \pref{sec:sequential}, the nature of the restriction that \eqref{eq:closeness_nonsequential} imposes on $\cF$ is less clear. 

We finish this introductory section by motivating the utility of the additional requirement \eqref{eq:closeness_nonsequential}. Once again, we only discuss the non-sequential case here.  Consider the following lower bound on Rademacher averages of a class $\cF$ in terms of $\vc(\cF,\alpha)$:
\begin{align}
    \label{eq:iid_rad}
    \En_{\veps}\sup_{f\in\cF} \sum_{i=1}^d \veps_i f(x_i) \geq \En_{\veps} \sum_{i=1}^d \veps_i (f^{\bepsilon}(x_i) -s_i) \geq n \alpha/2
\end{align}
where $d=\vc(\cF,\alpha)$ and $\{x_1,\ldots,x_d\}$ is a set whose existence is guaranteed by the definition. Here the expectation is with respect to independent Rademacher random variables $\veps_i$ (taking values $\pm1$ with equal probability). The lower bound argument can be extended to $d>\vc(\cF,\alpha)$ by considering repeated blocks of points and appealing to Khintchine's inequality. Such an argument leads to lower bounds of order $\Omega(\alpha\sqrt{\vc(\cF,\alpha)n})$, implying that the scale-sensitive dimension is an inherent barrier for Rademacher averages to be small, and, as a consequence, a barrier for certain learning problems.  Indeed, the notion of Radmacher averages in \eqref{eq:iid_rad} is known to be a key object in the study of prediction with i.i.d. data and ``non-curved'' losses. On the other hand, for loss functions such as square loss, it is the \textit{offset Rademacher averages}---or closely related local Rademacher averages \cite{bartlett2005local}---that in many situations correctly quantify the rates of convergence. The (non-sequential) offset Rademacher averages are defined as:\footnote{We replaced $\cF$ with $\cF-\cF$ to simplify the centering issues.}
\begin{align}
    \label{eq:iid_offset} 
    \En_{\veps}\sup_{f\in\cF-\cF} \sum_{i=1}^n \veps_i f(x_i) - c\cdot f(x_i)^2.
\end{align}
Unfortunately, the lack of control on the magnitudes of departures of $f^\bepsilon(x_i)$ from $s_i\pm \alpha/2$  prevents us from obtaining sufficiently strong lower bounds when considering a shattered set, as the negative term in \eqref{eq:iid_offset} may render the lower bound vacuous. This question motivates the study of gapped scale-sensitive dimensions, as presented in the next sections for both the non-sequential and sequential cases.

We briefly mention that a number of other versions of combinatorial dimensions have been proposed over the last few decades (see \cite{DanShaSh14, brukhim2022characterization,hanneke2023universal} and references therein). To the best of our knowledge, these notions are different from those proposed in the present paper, and do not immediately imply the lower bounds we seek.

\paragraph{Organization} We start the technical part of the paper with the non-sequential version of the gapped dimension in \pref{sec:non-sequential}.  We first introduce the gapped dimension for \emph{integer-valued} classes in \pref{sec:multiclass_nonseq} and state a version of the Sauer-Shelah-Vapnik-Chervonenkis lemma for this combinatorial definition due to \cite{anthony2000function},\footnote{After this paper was completed, we were informed by Peter Bartlett that the definition of the ``gapped'' dimension and Lemma~\ref{lem: non-sequential-integer} are contained in \cite{anthony2000function}.} yielding control of covering numbers. We then extend the definition to the real-valued case in \pref{sec:realvalued_nonseq} and prove that this scale-sensitive dimension controls covering numbers, similarly to the development in \cite{alon1997scale} for the standard definition. We then prove that offset Rademacher averages for any uniformly bounded class are lower bounded according to the behavior of the gapped scale-sensitive dimension of this class (\pref{sec:lower_bound_offset_nonsequential}), and present the ensuing lower bound for the problem of online transductive regression. \pref{sec:sequential} mirrors the development in \pref{sec:non-sequential} for the sequential case, with an application to online regression.

\paragraph{Notation} We denote $x_{1:d}=\{x_1,\ldots,x_d\}$ and $[M]=\{1,\ldots,M\}$ for integer $M>1$.For functions $A(\alpha, n)$ and $B(\alpha, n)$, we use $A(\alpha, n) = \Omega(B(\alpha, n))$ or $B(\alpha, n) = \mathcal{O}(A(\alpha, n))$ to denote $A(\alpha, n) \ge c\cdot B(\alpha, n)$ for any $\alpha > 0$ and positive integer $n$, with some fixed positive constant $c$. We use $A(\alpha, n) = \tilde{\Omega}(B(\alpha, n))$ or $B(\alpha, n) = \tO(A(\alpha, n))$ to denote $A(\alpha, n) \ge c \cdot B(\alpha, n) / \log^r(n/\alpha)$ holds for any $\alpha > 0$ and positive integer $n$, with some fixed positive constants $c, r$.

\section{Non-Sequential Gapped Dimensions}\label{sec:non-sequential}

\subsection{Integer-Valued Functions}
\label{sec:multiclass_nonseq}

Suppose $M$ is a positive integer. Let $\frakc: [M]\times [M]\to \RR_+\cup\{0\}$ be a distance metric. Consider the following combinatorial parameter \cite{anthony2000function}:
\begin{definition}[(Non-Sequential) Gapped Dimension]\label{def: non-sequential-finite}
	Let $\calF\subseteq  \{f: \calX\to [M]\}$ be a function class and fix $\alpha \ge 0$. We sat that $\calF$ shatters the set $\{x_1, \ldots, x_d\}\subseteq \calX$ at scale $\alpha$ if there exists $s_1 = (s_1[-1], s_1[1]), s_2 = (s_2[-1], s_2[1]), \ldots, s_d = (s_d[-1], s_d[1])\in [M]\times [M]$ with the following properties:
	\begin{enumerate}
		\item For any $t\in [d]$, $\frakc(s_t[1], s_t[-1]) \ge \alpha$;
		\item For any $\bepsilon\in \{-1, 1\}^d$, there exists $f^\bepsilon\in \calF$ such that $f^{\bepsilon}(x_t) = s_t[\epsilon_t]$ for any $t\in [d]$.
	\end{enumerate}
	We define the (non-sequential) gapped scale-sensitive dimension $\dnon_\frakc(\calF, \alpha)$ (or $\dnon(\calF, \alpha)$ when $\frakc$ is clear from context) of $\calF$ as the largest $d$ such that there exists $\{x_1, \ldots, x_d\}$ which is shattered by $\calF$.
\end{definition}

The gapped dimension in \pref{def: non-sequential-finite} was introduced in \cite{anthony2000function}. The definition is similar to that of the Natarajan dimension \citep{natarajan1988two,natarajan1989learning} for multi-class learning, with the important difference that the two ``labels'' (denoted by the choice $s_t[1]$ and $s_t[-1]$) are $\alpha$-separated (hence the name \emph{gapped} dimension); unlike multi-class problems where the the labels are treated as a categorical variable, in our case they are ordinal. 

We now recall the standard definition of a covering number, which we state here with respect to the distance $\frakc$ on each coordinate.
\begin{definition}[(Non-Sequential) Covering Number for Integer-Valued Functions]
	Fix $x_1, \ldots, x_n\in \calX$. We say that the set $\calV = \{\bv = (v_1, \ldots, v_n) \in [M]^n\}$, is a (non-sequential) cover of $\calF$ on $x_{1:n}$ at scale $\alpha$ if for any function $f\in \calF$, there exists $\bv\in \calV$ such that
	$$\max_{t\in [n]} \frakc(f(x_t), v_t)\le \alpha.$$
	We use $\calNnon_{\frakc, \infty} (\calF, x_{1:n}, \alpha)$ (or $\calNnon_\infty(\calF, x_{1:n}, \alpha)$ when $\frakc$ is clear from context) to denote the size of the smallest non-sequential cover of $\calF$ on $x_{1:n}$ at scale $\alpha$.
\end{definition}
The following lemma upper bounds the covering number of a integer-valued class by the gapped dimension.
\begin{lemma}[\cite{anthony2000function}]\label{lem: non-sequential-integer}
	Let $\calF\subseteq  \{f: \calX\to [M]\}$ and $\{x_1, \ldots, x_n\}\subseteq  \calX$. Then we have 
	$$\log \calNnon_\infty(\calF, x_{1:n}, \alpha)\le 16\dnon(\calF, \alpha)\log^2(enM)$$
\end{lemma}
The proof of \pref{lem: non-sequential-integer} is deferred to \pref{sec: non-sequential-integer-app}.
This result is similar to that of \cite{alon1997scale}, which provides an upper bound on the covering number in terms of the (original) scale-sensitive dimension. Since the gapped dimension can be smaller than the original definition, \pref{lem: non-sequential-integer} is an improvement over the corresponding result in \cite{alon1997scale}.

\subsection{Real-Valued Functions}
\label{sec:realvalued_nonseq}

We now turn our attention to the case of real-valued function classes. Let $\frakc: [-1, 1]\times [-1, 1]\to \RR_{+}\cup \{0\}$ be a distance metric; for example, we may choose $\frakc(a, b)=|a-b|$. We define the following notion of non-sequential gapped scale-sensitive dimension:
\begin{definition}[(Non-Sequential) Gapped Scale-Sensitive Dimension]\label{def: non-sequential-real}
	Let $\calF\subseteq  \{f:\cX\to[-1, 1]\}$ and fix $\alpha, \beta > 0$. We say that $\calF$ shatters $\{x_1, \ldots, x_d\}\subseteq \cX$ at scale $(\alpha, \beta)$ if there exist $s_1 = (s_1[-1], s_1[1]), s_2 = (s_2[-1], s_2[1]), \ldots, s_d = (s_d[-1], s_d[1])\in [-1, 1]\times [-1, 1]$ with the following properties:
	\begin{enumerate}
		\item For any $t\in [d]$, $\frakc(s_t[1], s_t[-1]) \ge \alpha$;
		\item For any $\bepsilon\in \{-1, 1\}^d$, there exists $f^\bepsilon\in \calF$ such that $\frakc(f^{\bepsilon}(x_t), s_t[\epsilon_t])\le \beta$ for any $t\in [d]$.
	\end{enumerate}
	We define the (non-sequential) gapped scale-sensitive dimension $\dnon(\calF, \alpha, \beta)$ of $\calF$ as the largest $d$ such that there exist $x_{1:d}\in \calX$ shattered by $\calF$ at scale $(\alpha, \beta)$.
\end{definition}

To illustrate the geometric requirement of the gapped scale-sensitive dimension, we refer to \pref{fig:shattering}. The standard definition of scale-sensitive dimension asks for a cube of side length $\alpha$ to be inscribed in the set, where ``inscribed'' means that there is an element of the set $\cF|_{x_1,\ldots,x_d}$ in any of the $2^d$ quadrants, for each vertex of the hypercube (formally, for any $\bepsilon\in\{-1,1\}^d$ there exists $f^{\bepsilon}\in\cF$ with $\veps_i(f^{\bepsilon}(x_i)-s_i)\geq \alpha/2$ for all $i\in \{1,\ldots,d\}$). In contrast, the gapped scale-sensitive dimension asks for a hypercube of side-length at least $\alpha$ to be inscribed in the set, where ``inscribed'' means that each of the $2^d$ vertices are $\beta$-close coordinate-wise to some element of $\cF|_{x_1,\ldots,x_d}$. While it is immediate that $\dnon(\calF, \alpha, \beta)\leq \vc(\cF,\alpha - 2\beta)$ for $\beta < \alpha/2$,
we show that this gap cannot be too large. We prove this fact by first establishing a relation between covering numbers and the gapped dimension.
\begin{figure}[h]
\centering
\includegraphics[width=0.6\textwidth]{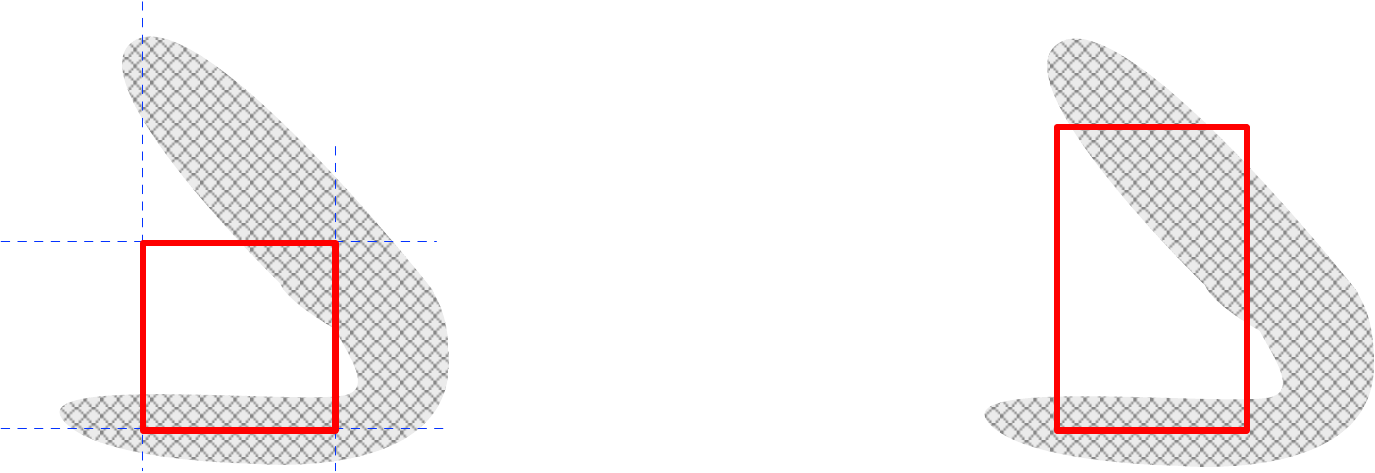}
\caption{Grey area depicts a set $\cF|_{x_1,\ldots,x_d}$. The hypercube on the left is ``inscribed'' in the classical sense, while the hyperrectangle on the right is ``inscribed'' according to the proposed definition.}
\label{fig:shattering}
\end{figure}

\begin{definition}[(Non-Sequential) Covering Number for Real-Valued Functions]\label{sec: def-nonsequential-covering}
	We say that a set $\calV \subseteq [-1, 1]^n$ is a cover of $\calF$ on $\{x_1,\ldots,x_n\}\subseteq \cX$ at scale $\alpha$ if for any function $f\in \calF$, there exists $v=(v_1,\ldots,v_n) \in \calV$ such that
	$$\max_{t\in [n]}\frakc(f(x_t), v_t)\le \alpha.$$
	We use $\calNnon_{\frakc, \infty}(\calF, x_{1:n}, \alpha)$ (or $\calNnon_\infty(\calF, x_{1:n}, \alpha)$ when $\frakc$ is clear from context) to denote the size of smallest sequential cover of $\calF$ on $x_{1:n}$ at scale $\alpha$. 
\end{definition}

\begin{proposition}\label{prop: non-covering-fat}
	For $\alpha, \beta > 0$, suppose there exists a positive integer $M$ and $M$ real numbers $-1\le u_1 < u_2 < \ldots < u_M\le 1$ such that for any $u\in [-1, 1]$, there exists some $i\in [M]$ such that $\frakc(u, u_i)\le \beta$. Then for any function class $\calF\subseteq \{f:\cX\to[-1, 1]\}$ and $\{x_1,\ldots,x_n\}\subseteq  \calX$, 
	$$\log \calNnon_\infty(\calF, x_{1:n}, \alpha + \beta)\le 16\dnon(\calF, \alpha, \beta)\log^2(enM).$$
    When $\frakc(a,b)=|a-b|$, a feasible value of $M$ is $\lfloor 2/\beta\rfloor$, which implies that
    $$\log \calNnon_\infty(\calF, x_{1:n}, \alpha + \beta)\le 16\dnon(\calF, \alpha, \beta)\log^2\left(\frac{2en}{\beta}\right).$$
\end{proposition}
A version of this result already appears as Lemma 4.3 in \cite{anthony2000function}, and we present the proof in \pref{sec: non-covering-fat-proof} for completeness.
We remark that the logarithmic dependence on $n$ is unavoidable. The question of reducing the power from $2$ to $1$ (with the classical definition of scale-sensitive dimension) has been studied in \cite{rudelson2006combinatorics}; we did not attempt to answer this question for the gapped dimension.

\subsection{Comparison to Scale-Sensitive Dimension}\label{sec: non-sequential-relation}
In this section, we compare the classic scale-sensitive dimension for real-valued function class and the non-sequential  scale-sensitive dimension defined in \pref{def: non-sequential-real}. We first recall the definition of the scale-sensitive dimension \cite{kearns1994efficient,bartlett1994fat}.
\begin{definition}\label{def: nonsequential-fat}
	Given a function class $\calF \subseteq  \{f: \calX\to [-1, 1]\}$, we say that $\calF$ shatters $\{x_1, \ldots, x_d\}\subseteq  \calX$ at scale $\alpha>0$ if there exists witnesses $s_1, \ldots, s_d\in [-1, 1]$ such that for any $\bepsilon = (\epsilon_{1:d})\in \{-1, 1\}^d$, there exists $f^\bepsilon\in \calF$ such that for all $t\in[n]$,
	$\epsilon_t\cdot (f^\bepsilon(x_t) - s_t)\ge \alpha/2.$ 
	The scale-sensitive dimension $\vc(\calF, \alpha)$ is defined to be the largest $d$ such that there exists $\{x_1,\ldots,x_d\}\subseteq  \calX$ shattered by $\calF$ at scale $\alpha$.
\end{definition}
We have the following relations between the non-sequential scale-sensitive dimension $\vc(\calF, \alpha)$ and the gapped dimension $\dnon(\calF, \alpha)$ with $\frakc(a, b) = |a - b|$. 
\begin{proposition}\label{prop: d-less-f-non}
	Given a function class $\calF\subseteq  \{\calX\to[-1,1]\}$, for any $0 < 2\beta < \alpha$, we have 
	$$\dnon(\calF, \alpha, \beta)\le \vc(\calF, \alpha - 2\beta).$$
\end{proposition}
\begin{proposition}\label{prop: f-less-d-non}
	Given a function class $\calF\subseteq  \{\calX\to[-1,1]\}$, for any $\alpha, \beta > 0$, we have
	$$\vc(\calF, 3(\alpha + \beta))\le 288\dnon(\calF, \alpha, \beta)\cdot \log^2\left(\frac{384\dnon(\calF, \alpha, \beta)}{\beta}\right).$$
    Furthermore, if $\calF$ is convex, then for any $\alpha, \beta > 0$ with $2\beta < \alpha$, 
	$$\vc(\calF, \alpha)\le \dnon(\calF, \alpha, \beta).$$
\end{proposition}
The proofs of \pref{prop: d-less-f} and \pref{prop: f-less-d} are almost identical to the proof of \cite[Theorem 4.2]{anthony2000function} and \cite[Theorem 4.3]{anthony2000function}. We defer both proofs to \pref{sec: f-d-proof-non}. We remark that \pref{prop: f-less-d-non} is proved \textit{using} \pref{prop: non-covering-fat}. We are not aware of a direct proof of \pref{prop: f-less-d-non}, which would, of course, allow us to re-use existing estimates of covering numbers via non-gapped dimensions.

There is an extra squared logarithmic factor in \pref{prop: f-less-d-non} compared to \pref{prop: d-less-f-non}. The following proposition indicates that at least one logarithmic factor in \pref{prop: f-less-d-non} is necessary.
\begin{proposition}\label{prop: log-necessary-non}
	There exists a class of contexts $\calX$, and a class of functions $\calF$, such that for any $0 < 2\beta < \alpha < 1$, we have 
	$$\dnon(\calF, \alpha, \beta) = 1,\quad \text{and}\quad \vc(\calF, \alpha) \ge \left\lfloor\log_2\left(\frac{1}{\alpha}\right)\right\rfloor.$$
\end{proposition}
The proof of \pref{prop: log-necessary-non} is deferred to \pref{sec: f-d-proof-non}.

\subsection{Lower Bounds for Offset Rademacher Complexity and Transductive Regression}
\label{sec:lower_bound_offset_nonsequential}

We fix some positive constant $C > 0$ and set of contexts $\calX$. For function class $\calF\subseteq \{\calX\to [-1, 1]\}$, $x_{1:n}\in \calX^n$, $\mu_{1:n}\in [-1, 1]^n$ and $\bepsilon\in \{-1, 1\}^n$, we define the supremum of the offset Rademacher process as
\begin{equation}\label{eq: offset-nonsequential}
	\calR_n(\calF, \mu_{1:n}, x_{1:n}, \bepsilon) = \sup_{f\in \calF}\sum_{t=1}^n C\cdot \epsilon_t (f(x_t) - \mu_t) - (f(x_t) - \mu_t)^2.
\end{equation}
The expectation $\En[\calR_n(\calF, \mu_{1:n}, x_{1:n}, \bepsilon)]$ with respect to i.i.d. Rademacher random variables $\bepsilon$ is termed the offset Rademacher complexity, and it is known to quantify the performance of Least Squares and related methods in regression.

\begin{theorem}\label{thm: offset-lower-bound-2}
	Suppose $C\ge 2$. Let $\frakc(a, b) = |a - b|$. If there exists $p\ge 0$ such that $\dnon(\calF, \alpha, \alpha/(20nC)) = \tOmega\left(\alpha^{-p}\right)$ holds for any $\alpha > 0$ and positive integer $n$, then for any positive integer $n$, 
	\begin{equation}\label{eq: nonsequential-lower-bound}
		\sup_{\mu_{1:n}, x_{1:n}} \EE[\calR_n(\calF, \mu_{1:n}, x_{1:n}, \bepsilon)] = \tOmega\left(n^{\frac{p}{p+2}}\right),
	\end{equation}
	where the expectation is with respect to i.i.d. Rademacher random variables $\bepsilon = (\epsilon_{1:n})\simiid \unif(\{-1, 1\})$.
\end{theorem}
The proof of \pref{thm: offset-lower-bound-2} is deferred to \pref{sec: offset-lower-bound-2-app}. 

\paragraph{Application: Transductive Regression}

In an $n$-round online transductive prediction problem, the forecaster is given a function class $\calF\subseteq \{\calX\to [-1, 1]\}$ and  contexts $\{x_1, \ldots, x_n\}\subseteq \calX$ before the interaction starts. On each round $t=1,\ldots,n$, the forecaster  makes prediction $\hy_t \in [-2, 2]$. Nature (or, adversary) then reveals the label $y_t\in [-2, 2]$. The forecaster's objective is to minimize the regret with respect to the performance of the best forecaster within the class $\calF$ in hindsight. Considering the square loss, we can write the optimal regret in this game as the following minimax value:
\begin{equation}\label{eq: transductive}
    \calV_n(\calF)\coloneqq \sup_{x_{1:n}\in \calX^n} \left\{\inf_{\hy_t}\sup_{y_t}\right\}_{t=1}^n \left[\sum_{t=1}^n \left(\hy_t - y_t\right)^2 - \inf_{f\in \calF} \sum_{t=1}^n\left(f(x_t) - y_t\right)^2\right],
\end{equation}
where $\{\cdot\}_{t=1}^n$ denotes repeated application of the operators. We remark that a typical assumption in transductive regression is that the set $\{x_1,\ldots,x_n\}$ is known, but not the order of appearance of its elements \cite{qian2024refined}. Such a setting is more difficult than the minimax game in \pref{eq: transductive}, as the forecaster has less information about contexts throughout the game. Hence, the lower bound for \pref{lem:transductive} also applies to this more widely studied setting.

The following theorem, a transductive analogue of \cite{rakhlin2014online}, lower bounds the regret  in \pref{eq: transductive} by the (non-sequential) offset Rademacher process \eqref{eq: offset-nonsequential}.
\begin{lemma}
    \label{lem:transductive}
	For any function class $\calF\subseteq \{\calX\to [-1, 1]\}$, we have the following lower bound on the transductive regression objective:
	$$\calV_n(\calF)\ge \sup_{x_{1:n}\in \calX}\sup_{\mu_{1:n}\in [-1, 1]} \EE_{\epsilon_t}\left[\sup_{f\in \calF} \sum_{t=1}^n 2 \epsilon_t(f(x_t) - \mu_t) - (f(x_t) - \mu_t)^2\right],$$
	where $\epsilon_{1:n}\simiid \unif(\{-1, 1\})$.
\end{lemma}
We defer the proof of \pref{lem:transductive} to \pref{sec: offset-lower-bound-2-app}. \pref{lem:transductive} together with \pref{thm: offset-lower-bound-2} and \pref{prop: non-covering-fat} implies the following lower bound for transductive online regression in terms of the non-sequential covering number defined in \pref{sec: def-nonsequential-covering}. Notably, the lower bound holds for any $\calF\subseteq \{f: \cX\to[-1,1]\}$. On the downside, 
the $[-2,2]$ range of predictions and outcomes does not correspond to the $[-1,1]$ range of the functions in $\cF$, making the problem misspecified in an atypical manner; this issue also appears in \cite[Lemma 4]{rakhlin2014online}, and we are not aware of other general proof techniques that circumvent this.
\begin{corollary}\label{corr: transductive}
	Fix a function class $\calF\subseteq \{\calX\to [-1, 1]\}$ over context space $\calX$. Suppose there exists real number $p\ge 0$ such that the $\ell_\infty$ covering number under distance $\frakc(a, b) = |a - b|$ satisfies $\sup_{x_{1:n}}\log \calNnon_\infty(\calF, x_{1:n}, \alpha) = \tOmega\left(\alpha^{-p}\right)$ for any $\alpha > 0$ and positive integer $n$. Then 
	$$\calV_n(\calF) = \tOmega\left(n^{\frac{p}{p+2}}\right).$$
\end{corollary}
In a manner similar to \cite[Lemma 9]{rakhlin2014online}, we can also show that $\calV_n(\calF)$ is lower bounded by $n^{\frac{p-1}{p}}$ up to constants, under the same setting of \pref{corr: transductive}. Hence, we obtain a complete picture of lower bounds for transductive regression in terms of the non-sequential covering numbers, modulo the misspecification issue discussed above.

\section{Sequential Gapped Dimensions}
\label{sec:sequential}

We recall that the sequential versions of aforementioned complexities are defined in terms of trees (or, equivalently, predictable processes). An $\cX$-valued tree $\bx$ of depth $n$ is a sequence of maps $x_1,\ldots,x_n$, with $x_t:\{-1, 1\}^{t-1}\to \cX$ and $x_1\in\cX$ a constant. We refer to $\bepsilon=(\epsilon_1,\ldots,\epsilon_n)\in\{-1, 1\}^n$ as a path of length $n$. Slightly abusing the notation, we write $x_t(\bepsilon)$ for $x_t(\bepsilon_{1:t-1})$, where, henceforth, $\bepsilon_{1:t-1}=(\epsilon_1,\ldots,\epsilon_{t-1})$. It is convenient to think of $\bx$ as a full binary tree labeled by elements of $\cX$. Similarly, we can define a tree labeled by $\reals$ or any other set. 
    
We recall that constant-level trees (those with $x_t(\bepsilon)=x_t \in\cX$ for all $t\in[n]$ and $\bepsilon$) correspond to a ``tuple'' of points $(x_1,\ldots,x_n)$. Likewise, sequential generalizations---such as sequential cover---reduce to the classical notions when the trees are constant-level \cite{rakhlin2015martingale}. However, we remark that the relations between the various sequential quantities do not imply the analogous relations in the non-sequential case. For this reason, in this paper we developed both sequential and non-sequential results separately.

\subsection{Integer-Valued Functions}
\label{sec:sequential_integervalued}

As before, let $\frakc: [M]\times [M]\to \RR_{+}\cup\{0\}$ be a distance metric. For a class of $[M]$-valued functions, we define the following notion of a sequential dimension:
\begin{definition}[Sequential Gapped Dimension]\label{def: shattering-finite}
    Let $\cF\subseteq  \{f:\cX\to[M]\}$ be a function class and fix $\alpha \geq 0$. We say that $\calF$ shatters an $\calX$-valued tree $\bx$ of depth $d$ at scale $\alpha$ if there exists an $[M]\times [M]$-valued tree $\bs$ of depth $d$ with the following properties: 
    \begin{itemize}
        \item[1.] For any $\bepsilon \in \{-1, 1\}^d$, $s_t(\bepsilon) = (s_t(\bepsilon)[-1], s_t(\bepsilon)[1])$ with $\frakc(s_t(\bepsilon)[1], s_t(\bepsilon)[-1]) \ge \alpha$.
        \item[2.] For any $\bepsilon\in \{-1, 1\}^d$, there exists $f^\bepsilon\in\calF$ such that $f^\bepsilon(x_t(\bepsilon))= s_t(\bepsilon)[\epsilon_t]$ for all $t\in[d]$.
    \end{itemize} 
    
		We define the gapped sequential scale-sensitive dimension $\dseq(\calF, \alpha)$ of $\calF$ as the largest $d$ such that there exists an $\alpha$-shattered tree $\bx$ of depth $d$.

\end{definition}

\begin{definition}[Sequential Covering Number for Integer-Valued Functions]
	Given an $\calX$-valued tree $\bx$ of depth $n$, we say that a set $\calV$ of $[M]$-valued trees of depth $n$ is a sequential cover of $\calF$ on $\bx$ at scale $\alpha$, if for any function $f\in \calF$ and $\bepsilon \in \{-1, 1\}^n$, there exists $\bv\in \calV$ such that for any $t\in [n]$, $\frakc(f(x_t(\bepsilon)), v_t(\bepsilon))\le \alpha$.

	We use $\calNseq_{\frakc, \infty}(\calF, \bx, \alpha)$ (or $\calNseq_\infty(\calF, \bx, \alpha)$ when $\frakc$ is clear from context) to denote the size of smallest sequential cover of $\calF$ on $\bx$ at scale $\alpha$.
\end{definition}

The following lemma upper bounds the sequential covering number for an integer-valued function class in terms of the sequential gapped dimension of the class, an analogue of \pref{lem: non-sequential-integer}.
\begin{lemma}\label{lem: shattering-finite}
	Let $\cF\subseteq  \{f:\cX\to[M]\}$ and let  $\bx$ be an $\calX$-valued tree of depth $n$. We have
	$$\log \calNseq_\infty(\calF, \bx, \alpha)\le \dseq(\calF, \alpha)\log \left(enM\right)$$
\end{lemma}
The proof of \pref{lem: shattering-finite} is deferred to \pref{sec: lower-bound-proof}.

\subsection{Real-Valued Functions}
Let $\frakc: [-1, 1]\times [-1, 1]\to \RR_{+}\cup\{0\}$ be a distance metric, as in \pref{sec:realvalued_nonseq}. We define the following notion of complexity of $\cF$, with respect to this metric:
\begin{definition}[Sequential Gapped Scale-sensitive Dimension]\label{def: seq-real-dim}
	Let $\calF\subseteq  \{f:\cX\to[-1, 1]\}$ be a function class and fix $\alpha, \beta > 0$. We say that $\calF$ shatters an $\calX$-valued tree $\bx$ of depth $d$ at scale $(\alpha, \beta)$ if there exists a $([-1, 1]\times [-1, 1])$-valued tree $\bs$ of depth $d$ with the following properties:   
    \begin{itemize}
        \item[1.] For any $\bepsilon \in \{-1, 1\}^d$, $s_t(\bepsilon) = (s_t(\bepsilon)[-1], s_t(\bepsilon)[1])$ with $\frakc(s_t(\bepsilon)[1], s_t(\bepsilon)[-1]) \ge \alpha$.
        \item[2.] For any $\bepsilon\in \{-1, 1\}^d$, there exists $f^\bepsilon\in\calF$ such that $\frakc(f^\bepsilon(x_t(\bepsilon)), s_t(\bepsilon)[\epsilon_t]) \leq \beta$.
    \end{itemize} 

	We define the gapped sequential scale-sensitive dimension $\dseq(\calF, \alpha, \beta)$ of $\calF$ as the largest $d$ such that there exists an $(\alpha, \beta)$-shattered tree $\bx$ of depth $d$.
\end{definition}

We now define sequential covering numbers and prove that their growth is controlled by the behavior of $\dseq$.

\begin{definition}[Sequential Covering Number for Real-Valued Functions]\label{def: seq-covering}
	Given an $\calX$-valued tree $\bx$ of depth $n$, we say that a set $\calV$ of $\reals$-valued trees of depth $n$ is a sequential cover of $\calF$ on $\bx$ at scale $\alpha$, if for any function $f\in \calF$ and $\bepsilon \in \{-1, 1\}^n$, there exists $\bv\in \calV$ such that for any $t\in [n]$, $\frakc(f(x_t(\bepsilon)), v_t(\bepsilon))\le \alpha$.

	We use $\calNseq_{\frakc, \infty}(\calF, \bx, \alpha)$ (or $\calNseq_\infty(\calF, \bx, \alpha)$ when $\frakc$ is clear from context) to denote the size of smallest sequential cover of $\calF$ on $\bx$ at scale $\alpha$.
\end{definition}

\begin{proposition}\label{prop: seq-covering-fat}
	For given real numbers $\alpha, \beta > 0$, suppose there exists a positive integer $M$ and $M$ real numbers $-1\le u_1 < u_2 < \ldots < u_M\le 1$ such that for any $u\in [-1, 1]$, there exists some $i\in [M]$ such that $\frakc(u, u_i)\le \beta$. Then for any function class $\calF\subseteq \{f:\cX\to[-1, 1]\}$, positive integer $n$ and $\alpha, \beta > 0$, for any depth-$n$ $\calX$-valued tree $\bx$ we have 
	$$\log \calNseq_\infty(\calF, \bx, \alpha + \beta)\le \dseq(\calF, \alpha, \beta)\log \left(enM\right).$$
    When $\frakc(a,b)=|a-b|$, a feasible value of $M$ is $\lfloor 2/\beta\rfloor$, which implies that
    $$\log \calNseq_\infty(\calF, x_{1:n}, \alpha + \beta)\le \dseq(\calF, \alpha, \beta)\log\left(\frac{2en}{\beta}\right).$$
\end{proposition}

The proof of \pref{prop: seq-covering-fat} is deferred to \pref{sec: lower-bound-proof}. The structure of the proof follows that in \cite{rakhlin2010online}; see also \cite{RakSri2015Steele} for a different sequential generalization of Natarajan and Steele dimensions. 

As in the previous works, \pref{prop: seq-covering-fat} relies (via discretization) on covering numbers and combinatorial dimensions for integer-valued function classes, as developed in \pref{sec:sequential_integervalued}.

\subsection{Comparison to Sequential Scale-Sensitive Dimension}\label{sec: f-d}
We recall the definition of sequential scale-sensitive dimension in \cite{rakhlin2015sequentialcomplexity}.
\begin{definition}[Sequential Scale-sensitive Dimension \cite{rakhlin2015sequentialcomplexity}]\label{def: sequential-fat}
	Given a set $\calX$ and a function class $\calF\subseteq \{\calX\to [-1, 1]\}$, a depth-$d$ $\calX$-valued tree $\bx$ is shattered by $\calF$ at scale $\alpha > 0$ if and only if there exists a depth-$d$ $[-1, 1]$-valued tree $\bv$ such that for any $\bepsilon\in\{-1, 1\}^n$, there exists $f^{\bepsilon}\in \calF$ which satisfies
	$$\epsilon_t\cdot \left(f^{\bepsilon}(x_t(\bepsilon)) - v_t(\bepsilon)\right) \ge \frac{\alpha}{2}.$$
	The tree $\bv$ is called the witness of the shattering. The sequential scale-sensitive dimension $\fat(\calF, \alpha)$ is defined to be the largest $d$ such that there exists depth-$d$ $\calX$-valued tree $\bx$ shattered by $\calF$ at scale $\alpha$.
\end{definition}

We have the following relations between the sequential scale-sensitive dimension $\fat(\calF, \alpha)$ in \pref{def: sequential-fat} and the gapped sequential scale-sensitive dimension $\dseq(\calF, \alpha)$ with the distance function $\frakc(a, b) = |a - b|$ in \pref{def: seq-real-dim}. First, observe that a tree that is $(\alpha,\beta)$-shattered in the above sense for $\beta \leq \alpha/4$ is also $\alpha/2$-shattered in the sense of the sequential scale-sensitive dimension $\fat$. Indeed, we may choose $\bar{\bs}$ as a $[-1, 1]$-valued tree defined by $\bar{s}_t(\bepsilon) = (s_t(\bepsilon)[1] + s_t(\bepsilon)[-1])/2$. We then have that for any $\bepsilon\in\{-1,1\}^d$, there exists $f^\bepsilon\in \cF$ such that $\epsilon_t(f^\bepsilon(x_t(\bepsilon))-\bar{s}_t(\bepsilon))\geq \alpha/4$. More precisely, we have the following:
\begin{proposition}\label{prop: d-less-f}
	Given a set $\calX$ and a function class $\calF\subseteq  \{\calX\to[-1,1]\}$, for any $0 < 2\beta < \alpha$, we have 
	$$\dseq(\calF, \alpha, \beta)\le \fat(\calF, \alpha - 2\beta)$$
\end{proposition}
For the reverse direction, the following holds:
\begin{proposition}\label{prop: f-less-d}
	Given set $\calX$ and function class $\calF\subseteq  \{\calX\to[-1,1]\}$, for any $\alpha, \beta > 0$, we have
	$$\fat(\calF, 3(\alpha + \beta))\le 4\dseq(\calF, \alpha, \beta)\log\left(\frac{12\dseq(\calF, \alpha, \beta)}{\beta}\right).$$
\end{proposition}
Just as in the non-sequential case, we remark that \pref{prop: f-less-d} is proved (in \pref{sec: f-d-proof}) \textit{using} \pref{prop: seq-covering-fat}, not the other way around. Furthermore, the following result says that the $\log(1/\beta)$ factor in \pref{prop: f-less-d} is indeed necessary.
\begin{proposition}\label{prop: log-necessary}
	Consider the case where $\calX = \{x\}$, and function class $\calF = \{f: \calX\to [-1, 1]: f(x)\in [-1, 1]\}$. Then for any $0 < 2\beta < \alpha < 1$, we have 
	$$\dseq(\calF, \alpha, \beta) = 1,\quad \text{and}\quad \fat(\calF, \alpha) \ge \left\lfloor\log_2\left(\frac{1}{\alpha}\right)\right\rfloor.$$
\end{proposition}
The proof of \pref{prop: log-necessary} is deferred to \pref{sec: f-d-proof}.

\subsection{Lower Bounds for Sequential Offset Rademacher Complexity and Online Regression }\label{sec: seq-offset-rademacher}

We fix some positive constant $C > 0$. For any set of contexts $\calX$, function class $\calF\subseteq \{\calX\to [-1, 1]\}$, depth-$n$ $\calX$-valued tree $\bx$, depth-$n$ $[-1, 1]$-valued tree $\bmu$ and $\{-1, 1\}$-path $\bepsilon\in \{-1, 1\}^n$, we define
$$\calR_n(\calF, \bmu, \bx, \bepsilon) = \sup_{f\in \calF}\sum_{t=1}^n \left\{C\cdot \epsilon_t (f(x_t(\bepsilon)) - \mu_t(\bepsilon)) - (f(x_t(\bepsilon)) - \mu_t(\bepsilon))^2\right\}.$$
The expected value $\EE[\calR_n(\calF, \bmu, \bx, \bepsilon)]$ is the sequential offset Rademacher complexity, known to govern the rates of online regression with squared and other strongly convex and smooth losses. We now show that this complexity is lower bounded, for any class $\cF$, by the scaling behavior of sequential scale-sensitive dimensions.

\begin{theorem}\label{thm: offset-lower-bound}
	Let $\frakc(a, b) = |a - b|$. If there exists a constant $p\ge 0$ such that $\dseq(\calF, \alpha, \alpha/20) = \tOmega\left(\alpha^{-p}\right)$ for any $\alpha > 0$, then for $C \ge 2$, we have
	$$\sup_{\bmu, \bx} \EE[\calR_n(\calF, \bmu, \bx, \bepsilon)] = \tOmega\left(n^{\frac{p}{p+2}}\right),$$
	where the supremum is over all depth-$n$ $\calX$-valued trees $\bx$ and depth-$n$ $[-1, 1]$-valued trees $\bmu$, and the expectation is with respect to i.i.d. Rademacher random variables $\bepsilon = (\epsilon_{1:n})\simiid \unif(\{-1, 1\})$.
\end{theorem}
The proof of \pref{thm: offset-lower-bound} is deferred to \pref{sec: lower-bound-proof}.

\paragraph{Application: Online Regression} In parallel to \pref{sec:lower_bound_offset_nonsequential}, \pref{thm: offset-lower-bound} implies the following lower bound (\pref{corr: sequential}) for online regression for any $\cF$, significantly improving upon the lower bound in \cite{rakhlin2014online}; the latter result only guaranteed \emph{existence} of $\cF$ with such lower bound properties. This improvement was the main motivation for this paper.

To formally state the lower bound, define the minimax value of the online prediction problem  $\calVseq_n(\calF)$ as
	$$\calVseq_n(\calF) \coloneqq\left\{\sup_{x_t\in \calX}\inf_{\hy_t}\sup_{y_t}\right\}_{t=1}^n \left[\sum_{t=1}^n \left(\hy_t - y_t\right)^2 - \inf_{f\in \calF} \sum_{t=1}^n\left(f(x_t) - y_t\right)^2\right].$$
Compared to the transductive setting in \pref{eq: transductive}, here the forecaster does not have access to the context $x_t$ until the beginning of round $t$. 

\begin{corollary}\label{corr: sequential}
	Fix a function class $\calF\subseteq \{\calX\to [-1, 1]\}$ over context space $\calX$. Suppose there exists a real number $p\ge 0$ such that the sequential covering number (defined in \pref{def: seq-covering}) under distance $\frakc(a, b) = |a - b|$ satisfies $\sup_{\bx}\log \calNseq_\infty(\calF, \bx, \alpha) = \tOmega\left(\alpha^{-p}\right)$, where the supremum is over all depth-$n$ $\calX$-valued trees. Then we have the following lower bound for the minimax regret:
	$$\calVseq_n(\calF) = \tOmega\left(n^{\frac{p}{p+2}}\right).$$
\end{corollary}

\section*{Acknowledgements} We acknowledge support from the Simons Foundation and the NSF through awards DMS-2031883 and PHY-2019786, as well as support from the DARPA AIQ program.

\bibliographystyle{alpha}
\bibliography{References.bib}

\newpage
\appendix

\section{Missing Proofs in \pref{sec:non-sequential}}
\subsection{Proofs of \pref{lem: non-sequential-integer}}\label{sec: non-sequential-integer-app}
\begin{proof}[Proof of \pref{lem: non-sequential-integer}]
	Without loss of generality we assume $\calX = \{x_1, x_2, \ldots, x_n\}$, since the largest subset of $\{x_1, \ldots, x_n\}$ shattered by $\calF$ is also a subset of $\calX$ shattered by $\calF$. In the following, we only need to consider cases where 
	\begin{equation}\label{eq: d-n-inequaltiy}
		4\dnon(\calF, \alpha)\log(eMn)\le n.
	\end{equation}
	In cases where \pref{eq: d-n-inequaltiy} fails, by taking the covering of all functions which map $\{x_{1:n}\}$ to $[M]$, we have the covering number estimate 
    $$\calNnon_\infty(\calF, x_{1:n}, \alpha)\le M^{n} = \exp(n\cdot \log M)\le \exp(4d(\calF, \alpha) \log(eMn)\cdot \log M) \le \exp(16d(\calF, \alpha)\log^2(eMn)).$$
	
	In the following, we let $d = d(\calF, \alpha)$. We will use an approach similar to that in \cite{alon1997scale} to prove \pref{lem: non-sequential-integer}. We define the packing number $\calM_\infty(\calF, x_{1:n}, \alpha)$ of class $\calF$ under design $x_{1:n}$ at scale $\alpha$: we say $\tcalF\subseteq \calF$ is a packing if for any $f, f'\in \tcalF$, 
	$$\max_{t\in [n]} \frakc(f(x_t), f'(x_t)) \ge \alpha,$$
	and we let $\calM_\infty(\calF, x_{1:n}, \alpha)$ be the size of the largest packing of class $\calF$. Then according to covering-packing inequality we have 
	$$\calNnon_\infty(\calF, x_{1:n}, \alpha)\le \calM_\infty(\calF, x_{1:n}, \alpha).$$
	Hence, it suffices to prove
	\begin{equation}\label{eq: packing-shattering}
		\log \calM_\infty(\calF, x_{1:n}, \alpha)\le 16d\log^2(eMn).
	\end{equation}
	
	For set $X = (x_1, \ldots, x_l)\subseteq \{x_{1:l}\}$, and tuple $\bs = (s_1, s_2, \ldots, s_l)$ where $s_t = (s_t[-1], s_t[1])\in [M]\times [M]$ for any $t\in [l]$, we say the pair $(X, \bs)$ is shattered by $\calF$ if the following two properties both hold:
	\begin{enumerate}[label=(\alph*)]
		\item for any $t\in [l]$, $\frakc(s_t[-1], s_t[-1]) \ge \alpha$;
		\item for any $\bepsilon = (\epsilon_{1:l})\in \{-1, 1\}^l$, there exists $f^\bepsilon\in \calF$ such that $f^{\bepsilon}(x_t) = s_t[\epsilon_t]$ for any $t\in [l]$.
	\end{enumerate}
	We define function 
	\begin{align*} 
		t(h, l)\coloneqq & \min_{\tcalX\subseteq \calX, |\tcalX| = l}\max\{k: \forall\text{ } F\subseteq \calF \text{ such that }|F| = h\text{ and } \forall f, f'\in F,\ \max_{x\in \tcalX} \frakc(f(x), f'(x)) \ge \alpha,\\
		&\qquad \qquad F \text{ shatters at least }k\text{ pairs }(X, \bs)\text{ where }X\subseteq \tcalX\}. \numberthis \label{eq: def-t-function}
	\end{align*}
	(if no packing of size $h$ exists, $t(h, l)$ is defined to be infinity). According to the definition of $d = \dnon(\calF, \alpha)$, if $(X, \bs)$ is shattered by $\calF$, then we have $|X|\le d$. Additionally, for fixed $X$, there can be at most $(M^2)^{|X|}$ choices of $\bs$ such that $(X, \bs)$ is shattered by $\calF$. Therefore, by choosing $h$ to be the size of largest packing, i.e. $\calM_\infty(\calF, x_{1:n}, \alpha)$, the number of pairs $(X, \bs)$ shattered by the largest packing is no more than the number of pairs $(X, \bs)$ shattered by $\calF$, which implies that
	\begin{equation}\label{eq: t-inequality-1}
		t(\calM_\infty(\calF, x_{1:n}, \alpha), n)\le \sum_{i=0}^d \binom{n}{i}\cdot M^{2i}.
	\end{equation}
	It is sufficient to argue that for any  $r + 1\le l\le n$, 
	\begin{equation}\label{eq: non-sequential-t-covering}
		t(2(2lM^2)^r, l) > 2^r.
	\end{equation}
	Indeed, with \pref{eq: non-sequential-t-covering} holding, it is easy to see from the definition of function $t$ in \pref{eq: def-t-function} that for $h_1\le h_2$ and any $l$, $t(h_1, l)\le t(h_2, l)$. Notice that according to \pref{eq: d-n-inequaltiy},
	$$n\ge 4d\log(eMn)\ge \log_2\left\lfloor \sum_{i=1}^d \binom{n}{i}M^{2i}\right\rfloor + 2.$$
	Hence, if
	$$\calM_\infty(\calF, x_{1:n}, \alpha) > 2(2nM^2)^{\log_2\lfloor \sum_{i=1}^d \binom{n}{i}M^{2i}\rfloor + 1},$$
	we have 
	\begin{align*} 
		t(\calM_\infty(\calF, x_{1:n}, \alpha), n) & \ge t\left(2(2nM^2)^{\log_2\lfloor \sum_{i=1}^d \binom{n}{i}M^{2i}\rfloor + 1}, n\right)\\
		& \ge 2^{\log_2\lfloor \sum_{i=1}^d \binom{n}{i}M^{2i}\rfloor + 1} > \sum_{i=1}^d \binom{n}{i}M^{2i},
	\end{align*}
	which contradicts \pref{eq: t-inequality-1}. Therefore, 
	$$\log \calNnon_\infty(\calF, x_{1:n}, \alpha)\le \log\calM_\infty(\calF, x_{1:n}, \alpha)\le \log\left(2(2nM^2)^{\log_2\lfloor \sum_{i=1}^d \binom{n}{i}M^{2i}\rfloor + 1 }\right)\le 16d\log^2(eMn). $$

	In the remainder of the proof, we argue that  \pref{eq: non-sequential-t-covering} holds by induction. We will verify the following two properties:
	\begin{align*}
		t(2, l)\ge 1 & \qquad \forall l\ge 1,\numberthis \label{eq: t-covering-property-1}\\
		t(2lM^2\cdot 2m, l)\ge 2\cdot t(2m, l-1) & \qquad \forall m\ge 1, l\ge 2. \numberthis \label{eq: t-covering-property-2}
	\end{align*}
	We first verify \pref{eq: t-covering-property-1}. For $\tcalX\subseteq \calX$ and any packing $F\subseteq \calF$ with $|F| = 2$, there exists $f_1, f_2\in F$ such that 
	$$\max_{x\in \tcalX} \frakc(f_1(x), f_2(x)) \ge \alpha.$$
	We let $\tx\in \tcalX$ to be the one which takes the maximum in the above inequality, and let $X = (\tx)\subseteq \calX$ and $\bs = \{s_1\}$ with $s_1 = (f_1(x), f_2(x))\in [M]\times [M]$, then $(X, \bs)$ is shattered by $\calF$. Therefore, $t(2, l)\ge 1$. 
 
	We next verify \pref{eq: t-covering-property-2}. Without loss of generality we assume the minimum over $\tcalX\subseteq \calX$ with $|\tcalX| = l$ in \pref{eq: def-t-function} is achieved by $\tcalX = \{x_{1:l}\}$. Suppose packing $F\subseteq \calF$ satisfies $|F|\ge 4lM^2\cdot m$. We arbitrarily pair up functions in $F$ to form:
	$$F = \bigcup_{t=1}^{2lM^2m} \{f_t, g_t\}.$$
	For any $t\in [2lM^2m]$, since $F$ is a packing, there exists $x[t]\in \{x_{1:l}\}$ such that 
	$$\frakc(f_t(x[t]), g_t(x[t])) \ge \alpha.$$
	Next, for any $x\in \{x_{1:l}\}$ and $i, j\in [M]$ with $\frakc(i, j) \ge \alpha$, we define the set 
	$$\calT(x, i, j) = \{t\in [2lM^2m]: x[t] = x\text{ and }f_t(x[t]) = i,\ g_t(x[t]) = j\}.$$
	Then we have 
	$$\bigcup_{\substack{x\in \{x_{1:l}\}\\i, j\in M, \frakc(i, j)\ge \alpha}} \calT(x, i, j) = [2lM^2m].$$
	According to the pigeonhole principle, there exists some $x_{t^\star}\in \{x_{1:l}\}$ and $i^\star, j^\star\in [M]$ with $\frakc(i^\star, j^\star) \ge \alpha$ such that $|\calT(x_{t^\star}, i^\star, j^\star)|\ge 2m$. We define two function classes $F_1, F_2\subseteq  F$ as 
	$$F_1 = \{f_t: t\in \calT(x_t^\star, i^\star, j^\star)\}\quad \text{and}\quad F_2 = \{g_t: t\in \calT(x_t^\star, i^\star, j^\star)\}.$$
	Then we have $|F_1| = |F_2|\ge 2m$. Since functions in $F_1$ all take value $i^\star$ at $x_{t^\star}$, $F_1$ is a packing under design $\{x_{1:l}\}\backslash \{x_{t^\star}\}$. Hence, there exists a set $V$ consisting of pairs $(X, \bs)$ shattered by class $F_1$, and $|V|\ge t(2m, l-1)$. Since functions in $F_1$ takes the same value at $x_{t^\star}$, for any $(X, \bs)\in V$, $x_{t^\star}\not\in X$. Similarly, there exists a set $U$ shattered by $F_2$ with $|U|\ge t(2m, l-1)$, and for any $(X, \bs)\in U$, $x_{t^\star}\not\in X$. Since $F_1\subseteq  F$ and $F_2\subseteq  F$, any pairs $(X, \bs)$ in $U\cup V$ are also shattered by $F$. Additionally, for any pair $(X, \bs)\in U\cap V$, we construct a new pair $(X\cup (x_{t^\star}), \bs\cup (i^\star, j^\star))\not\in U\cup V$. Since $\frakc(i^\star, j^\star) \ge \alpha$, this pair is also shattered by $F$. Hence $F$ shatters at least 
	$$|U\cap V| + |U\cup V| = |U| + |V| = 2t(2m, l-1)$$
	pairs, which implies that $t(4lM^2m, l)\ge 2t(2m, l-1)$. This completes the proof of \pref{eq: t-covering-property-2}.

	With \pref{eq: t-covering-property-1} and \pref{eq: t-covering-property-2}, when $r\le l-1$ we have 
	$$t(2 (2lM^2)^r, l)\ge 2^r t(2, l-r)\ge 2^r.$$
	which verifies \pref{eq: non-sequential-t-covering}.
\end{proof}

\subsection{Proof of \pref{prop: non-covering-fat}}\label{sec: non-covering-fat-proof}
\begin{proof}[Proof of \pref{prop: non-covering-fat}]
	We define distance $\frakc': [M]\times [M]\to \RR_+\cup\{0\}$:
	\begin{equation}\label{eq: def-cprime-c-1}
		\frakc'(a, b) = \frakc(u_a, u_b).
	\end{equation}
	For any $f\in \calF$, since $f$ maps $\calX$ into $[-1, 1]$, we define $\bar{f}: \calX\to [M]$ to be an arbitrary function mapping $\calX$ into $[M]$, which satisfies that for any $x\in \calX$, $\frakc(f(x), u_{\barf(x)})\le \beta$. We construct function class $\barF = \{\barf: f\in \calF\}\subseteq \{\calX\to [M]\}$. We use $\dnon_{\frakc'}(\barF, \alpha)$ to denote the non-sequential gapped dimension (defined in \pref{def: non-sequential-finite}) of integer-valued function class $\barF$ at scale $\alpha$ under distance $\frakc'$, and for simplicity we let $d = \dnon_{\frakc'}(\barF, \alpha)$. Suppose $x_{1:d}\in \calX^d$ is shattered by $\calF$. Then there exists $\bars_1 = (\bars_1[-1], \bars_1[1]), \ldots, \bars_d = (\bars_d[-1], \bars_d[1])$ such that for any $\bepsilon\in \{-1, 1\}^d$ and $t\in [d]$, $\frakc'(\bars_t[-1], \bars_t[1]) \ge \alpha$, and also for any $\bepsilon\in \{-1, 1\}^d$, there exists $\barf^{\bepsilon}\in \barF$ such that $\barf^\bepsilon(x_t) = \bars_t[\epsilon_t]$ for any $t\in [d]$. Notice from the definition of $\calF$ that there exists some $f^\bepsilon\in \calF$ such that $\barf^\bepsilon = \overline{(f^\bepsilon)}$.

	We next verify that $x_{1:n}$ is also shattered by $\calF$ at scale $(\alpha, \beta)$, according to the definition \pref{def: non-sequential-real}. We construct $s_1 = (s_1[-1], s_1[1]), \ldots, s_d = (s_d[-1], s_d[1])$ according to $\bars_{1:d}$ as follows:
	$$s_t[-1] = u_{\bars_t[-1]}\quad \text{and}\quad s_t[1] = u_{\bars_t[1]}.$$ 
	Then according to the definition of $\frakc'$ in \pref{eq: def-cprime-c-1}, we have for any $t\in [d]$,
	$$\frakc(s_t[-1], s_t[1]) = \frakc'(\bars_t[-1], \bars_t[1]) \ge \alpha.$$
	Additionally, for any $\bepsilon\in \{-1, 1\}^d$, there exists some $f^\bepsilon\in \calF$ such that $\overline{(f^\bepsilon)}(x_t) = \bars_t[\epsilon_t]$ for any $t\in [d]$, which implies that 
	$$\frakc(f(x_t), s_t[\epsilon_t]) = \frakc(f(x_t), u_{\bars_t[\epsilon_t]}) = \frakc(f(x_t), u_{\overline{(f^\bepsilon)}(x_t)})\le \beta,\quad \forall t\in [d],$$
	where the last inequality follows from the definiton of $\bar{f}$. Therefore, $x_{1:d}$ is shattered by $\calF$ at scale $(\alpha, \beta)$, and, according to \pref{def: non-sequential-real}, we have $\dnon(\calF, \alpha, \beta)\ge d$. 

	Next, we will upper bound the non-sequential covering number of $\calF$ in terms of $d$. According to \pref{lem: non-sequential-integer}, for any $x_{1:n}\in \calX^n$ there exists a non-sequential covering $\barV$ of $\barF$ with size no more than $\exp\left(16\dnon\log^2(enM)\right)$. Hence, for any $f\in \calF$, there exists some $\barbv = (\barv_{1:n})\in \barV$ which satisfies
	$$\frakc'(\barf(x_t), \barv_t)\le \alpha\qquad \forall t\in [n],$$
	which implies that 
	$$\frakc(f(x_t), u_{\barv_t})\le \frakc(f(x_t), u_{\bar{f}(x_t)}) + \frakc(u_{\bar{f}(x_t)}, u_{\barv_t})\le \beta + \alpha\qquad\forall t\in [n],$$
	where the first inequality uses the triangle inequality of $\frakc$, and the second inequality uses the definition of function $\bar{f}$. For every $\barbv\in \barV$, we construct $\bv^{\barbv} = (v_{1:n}^{\barbv})\in [-1, 1]^n$: for any $t\in [n]$, $v^{\barbv}_t = u_{\barv_t}$. We further let $\calV = \{\bv^{\barbv}: \barbv\in \barV\}$. Then $\calV$ is an $(\alpha + \beta)$-cover of $\calF$ on $x_{1:n}$, which implies 
	$$\log \calNnon_\infty(\calF, \bx, \alpha + \beta)\le \log|\calV| = \log|\barV|\le 16\dnon\log^2(enM)\le 16\dnon(\calF, \alpha, \beta)\log^2(enM).$$
\end{proof}

\subsection{Missing Proofs in \pref{sec: non-sequential-relation}}\label{sec: f-d-proof-non}
\begin{proof}[Proof of \pref{prop: d-less-f-non}]
	We only need to prove that if $x_{1:d}\in \calX^d$ is shattered by $\calF$ at scale $(\alpha, \beta)$ according to \pref{def: non-sequential-real}, then $x_{1:d}$ is shattered by $\calF$ at scale $\alpha - 2\beta$ according to the traditional notion of shattering as in \pref{def: nonsequential-fat}.

	If $x_{1:d}$ is shattered by $\calF$ at scale $(\alpha, \beta)$ according to \pref{def: non-sequential-real}, then there exists $s_{1:d}$ where $s_t = (s_t[-1], s_t[1])\in [-1, 1]\times [-1, 1]$ for any $t\in [d]$, such that $s_t[1] - s_t[-1] \ge \alpha$ for any $t\in [d]$, and also for any $\bepsilon\in \{-1, 1\}^d$, there exists a function $f^{\bepsilon}\in \calF$ which satisfies $|f^\bepsilon(x_t) - s_t[\epsilon_t]|\le \beta$ for any $t\in [d]$. We let 
	$$v_t = \frac{s_t[-1] + s_t[1]}{2},\quad \forall t\in [n].$$
	Since $\alpha > 2\beta$, we have for any $\bepsilon\in \{-1, 1\}^d$, 
	$$\epsilon_t\cdot \left(f^{\bepsilon}(x_t) - v_t\right)\ge \frac{\alpha}{2} - \beta,$$
	which implies that $x_{1:d}$ is shattered by $\calF$ at scale $\alpha - 2\beta$, according to \pref{def: nonsequential-fat}.
\end{proof}
\begin{proof}[Proof of \pref{prop: f-less-d-non}]
	Let $d = \vc(\calF, \alpha)$, then there exists $x_{1:d}$ shattered by $\calF$ according to \pref{def: nonsequential-fat}. Hence there exists $v_{1:d}\in [-1, 1]^d$ and also function $f^\bepsilon\in \calF$ for each $\bepsilon\in \{-1, 1\}^d$, such that for any $\bepsilon\in \{-1, 1\}^d$, 
	$$\epsilon_t\cdot (f^\bepsilon(x_t) - v_t)\ge \frac{\alpha}{2}.$$
	Further notice for every $\bepsilon\neq \bepsilon'$, there exists $t\in [d]$ such that $\epsilon_t\neq \epsilon'_t$, which implies that for this specific $t$,
	$$|f^\bepsilon(x_t) - f^{\bepsilon'}(x_t)|\ge \alpha.$$
	Hence $\{f^\bepsilon: \bepsilon\in \{-1, 1\}^d\}$ forms a $\alpha$-packing of $\calF$ under $\ell_\infty$-norm under design $x_{1:d}$. Therefore, the covering-packing duality indicates that
	\begin{equation}\label{eq: N-lower-bound-non}
		\calNnon_\infty(\calF, x_{1:\vc(\calF, \alpha)}, \alpha/3)\ge 2^d = 2^{\vc(\calF, \alpha)}.
	\end{equation}

	Next according to \pref{prop: non-covering-fat}, we have for any $n$ and $x_{1:n}\in \calX^n$,  
	$$\calNnon_\infty(\calF, x_{1:n}, \alpha + \beta)\le 16\dnon(\calF, \alpha, \beta)\log^2\left(\frac{2en}{\beta}\right).$$
	Hence by replacing $\alpha$ in \pref{eq: N-lower-bound-non} with $3(\alpha + \beta)$, and choosing $n$ to be $\vc(\calF, 3(\alpha + \beta))$, we obtain that
	\begin{align*} 
            \log 2\cdot \vc(\calF, 3(\alpha + \beta)) & \le \sup_{x_{1:\vc(\calF, \alpha + \beta)}}\log\calNnon_\infty(\calF, x_{1:\vc(\calF, 3(\alpha + \beta))}, \alpha + \beta)\\
            & \le 16\dnon(\calF, \alpha, \beta)\log^2\left(\frac{2e\cdot \vc(\calF, 3(\alpha + \beta))}{\beta}\right),
        \end{align*}
	which implies 
	\begin{equation}\label{eq: non-seq-vc-d}
            \vc(\calF, 3(\alpha + \beta))\le 32\dnon(\calF, \alpha, \beta)\log^2\left(\frac{6\vc(\calF, 3(\alpha + \beta))}{\beta}\right).
        \end{equation}
    Additionally, we notice that $\log(x) \le \sqrt{2}\cdot \sqrt[3]{x}$, hence 
    $$\vc(\calF, 3(\alpha + \beta))\le 64\dnon(\calF, \alpha, \beta)\left(\frac{6\vc(\calF, 3(\alpha + \beta))}{\beta}\right)^{2/3},$$
    which implies that 
    $$\vc(\calF, 3(\alpha + \beta))\le 64^3\cdot 6^2\cdot \frac{\dnon(\calF, \alpha, \beta)^3}{\beta^2}.$$
    Bringing this back to \pref{eq: non-seq-vc-d}, we obtain that 
    $$\vc(\calF, 3(\alpha + \beta))\le 288\dnon(\calF, \alpha, \beta)\cdot \log^2\left(\frac{384\dnon(\calF, \alpha, \beta)}{\beta}\right).$$
            
    The second part of \pref{prop: f-less-d-non} is implied by \pref{prop: fat-equivalence} below. Indeed, if $x_{1:d}$ is shattered by $\calF$ at scale $\alpha$ according to \pref{def: fixed-scale dim}, then $x_{1:d}$ is also shattered by $\calF$ at scale $(\alpha, \beta)$ according to \pref{def: non-sequential-real}.
\end{proof}

Finally, we provide the proof of \pref{prop: log-necessary-non}.
\begin{proof}[Proof of \pref{prop: log-necessary-non}]
	Let $d = \lfloor \log(1/\alpha)\rfloor$ and $\calX = \{x_1, \ldots, x_d\}$. For any $\bepsilon = (\epsilon_{1:d})\in \{-1, 1\}^d$, we define 
	$$a^\bepsilon = \alpha + \sum_{i=1}^d \frac{\epsilon_i + 1}{2}\cdot 2^{i-1}\cdot \alpha.$$
	Then for any $\bepsilon$, we have 
	$$\alpha\le a^\bepsilon\le \alpha + \alpha\cdot (2^d - 1)\le 1.$$
	We define function class $\calF = \{f^\bepsilon: \bepsilon\in \{-1, 1\}^d\}$, where $f^\bepsilon: \calX\to [-1, 1]$ is defined as
	$$f^\bepsilon(x_i) = \epsilon_i\cdot a^\bepsilon,\quad \forall 1\le i\le d.$$
	Then it is easy to see that $\{x_1, \ldots, x_d\}$ is shattered by $\calF$ in terms of the classical shattering (defined in \pref{def: nonsequential-fat}).

	Next, we will verify that the non-sequential gapped dimension $d(\calF, \alpha, \beta) = 1$ for any $\beta < \alpha/2$. First of all, it is easy to see that $\{x_1\}$ is shattered by $\calF$, hence $d(\calF, \alpha, \beta) \ge 1$. If there exists a size-two subset $\{x_i, x_j\}$ of $\calX$ shattered by $\calF$, in terms of \pref{def: non-sequential-finite}, then there exists $s_i = (s_i[-1], s_i[1])\in [-1, 1]\times [-1, 1]$ and $s_j = (s_j[-1], s_j[1])\in [-1, 1]\times [-1, 1]$ such that for any $\be = (e_i, e_j)\in \{-1, 1\}\times \{-1, 1\}$, 
	$$\exists\ \bepsilon[\be]\in \{-1, 1\}^d\quad \text{such that}\quad |f^{\bepsilon[\be]}(x_i) - s_i[e_i]|\le \beta\quad \text{and}\quad |f^{\bepsilon[\be]}(x_j) - s_j[e_j]|\le \beta.$$
	Hence, we obtain 
	$$|f^{\bepsilon[(-1, -1)]}(x_i) - f^{\bepsilon[(-1, 1)]}(x_i)|\le 2\beta < \alpha.$$
	According to the construction of function $f^\bepsilon$, we must have $\bepsilon[(-1, -1)] = \bepsilon[(-1, 1)]$, which implies that 
	$$|s_j[1] - s_j[-1]|\le |s_j[1] - f^{\bepsilon[(-1, 1)]}(x_j)| + |s_j[-1] - f^{\bepsilon[(-1, -1)]}(x_j)|\le \beta + \beta < \alpha.$$
	This violates the definition of shattering in \pref{def: non-sequential-finite}. Therefore, we have verified that $d(\calF, \alpha, \beta)\le 1$ for any $\beta < \alpha/2$.
\end{proof}

\subsection{Properties of Fixed-Scale Scale-Sensitive Dimension}
We consider the following fixed-scale scale-sensitive dimension; the only modification with respect to \pref{def: nonsequential-fat} is that the inequality is turned into an equality. In terms of \pref{fig:shattering}, the requirement of shattering states that the vertices of the hypercube with side-length $\alpha$ are in the set.
\begin{definition}[Fixed-Scale Dimension]\label{def: fixed-scale dim}
	Given a function class $\calF\subseteq  \{\calX\to [-1, 1]\}$, we say that a set $\{x_1, x_2, \ldots, x_d\}\subseteq \calX$ is shattered by $\calF$ at scale $\alpha > 0$, if there exists $s_1, \ldots, s_d\in [-1, 1]$ such that for any $\bepsilon = (\epsilon_{1:d})\in \{-1, 1\}^d$, there exists some $f^\bepsilon\in \calF$ such that 
	$$\epsilon_t\cdot (f^\bepsilon(x_t) - s_t) = \frac{\alpha}{2},\qquad \forall t\in [d].$$
	The fixed-scale dimension $\vc_\fix(\calF, \alpha)$ of $\calF$ at scale $\alpha$ is the largest $d$ such that there exists a size-$d$ subset of $\calX$ shattered by $\calF$. 
\end{definition}

We have the following proposition showing that if the function class $\calF$ is convex, then the scale-sensitive dimensions defined in \pref{def: nonsequential-fat} and in \pref{def: fixed-scale dim} coincide.
\begin{proposition}\label{prop: fat-equivalence}
	If function class $\calF$ is convex, then for any $\alpha > 0$, $\vc(\calF, \alpha) = \vc_\fix(\calF, \alpha)$.
\end{proposition}

\begin{proof}[Proof of \pref{prop: fat-equivalence}]
	According to \pref{def: nonsequential-fat} and \pref{def: fixed-scale dim}, we have $\vc(\calF, \alpha)\ge \vc_\fix(\calF, \alpha)$. Hence we only need to prove that $\vc_\fix(\calF, \alpha)\ge \vc(\calF, \alpha)$. In the following, we will show that if $\{x_1, \ldots, x_d\}$ is shattered by $\calF$ at scale $\alpha$ under witness $s_1, \ldots, s_d$ under \pref{def: nonsequential-fat}, then  $\{x_1, \ldots, x_d\}$ is shattered by $\calF$ at scale $\alpha$ under witness $s_1, \ldots, s_d$ under \pref{def: fixed-scale dim}.

	Since $\{x_1, \ldots, x_d\}$ is shattered by $\calF$ at scale $\alpha$ under witness $s_1, \ldots, s_d$ under \pref{def: nonsequential-fat}, for any $\bepsilon\in \{-1, 1\}^d$ there exists $f^\bepsilon$ such that
	$$\epsilon_t\cdot \left(f^{\bepsilon}(x_t) - s_t\right)\ge \frac{\alpha}{2}\qquad \forall t\in [d].$$\
	We next iteratively construct $f^\bepsilon_i\in \calF$ for $i = 0, 1, \ldots, d$ such that $f^{\bepsilon}_i$ satisfies
	\begin{align*} 
		\epsilon_t\cdot (f^{\bepsilon_i}(x_t) - s_t) = \frac{\alpha}{2}\quad \text{if }t\le i\qquad \text{and}\qquad\epsilon_t\cdot (f^{\bepsilon_i}(x_t) - s_t) \ge \frac{\alpha}{2} \quad \text{if }t > i. \numberthis\label{eq: f-x-i}
	\end{align*}
	For $i = 0$, we let $f_0^\bepsilon = f^\bepsilon$ and they satisfies conditions in \pref{eq: f-x-i}. Suppose we have constructed $f^\bepsilon_i$, and we will now construct $f^\bepsilon_{i+1}$. For any $\bepsilon = (\epsilon_{1:d})\in \{-1, 1\}^d$, we let 
	\begin{equation}\label{eq: def-btepsilon}
		\btepsilon = (\epsilon_1, \ldots, \epsilon_i, -\epsilon_{i+1}, \epsilon_{i+2}, \ldots, \epsilon_d).
	\end{equation}
	Let
	\begin{equation}\label{eq: f-i+1}
		f^{\bepsilon}_{i+1} = \alpha^\bepsilon\cdot f^{\bepsilon}_i + (1 - \alpha^\bepsilon)\cdot f^{\btepsilon}_i,
	\end{equation}
	where 
	$$\alpha^\bepsilon = \begin{cases}
		\frac{\alpha/2 + s_{i+1} - f^{\btepsilon}(x_{i+1})}{f^{\btepsilon}_i(x_{i+1}) - f^{\bepsilon}_i(x_{i+1})} &\quad \text{if }\epsilon_{i+1} = 1,\\
		\frac{\alpha/2 + f^{\btepsilon}(x_{i+1}) - s_{i+1}}{f^{\btepsilon}_i(x_{i+1}) - f^{\bepsilon}_i(x_{i+1})} &\quad \text{if } \epsilon_{i+1} = -1.
	\end{cases}$$
	According to \pref{eq: f-x-i}, we can verify that $\alpha^{\btepsilon}\in [0, 1]$ for any $\bepsilon$. Therefore $f^{\bepsilon}_{i+1}$ constructed in \pref{eq: f-i+1} belongs to $\calF$ due to the convexity of $\calF$. Next, we will verify that $f^{\bepsilon}_{i+1}$ also satisfies \pref{eq: f-x-i}. For $t \le i$, since
	$$f^\bepsilon_i(x_t) = s_t + \epsilon_t\cdot \frac{\alpha}{2}$$
	according to \pref{eq: f-x-i}, according to our choice of $\btepsilon$ in \pref{eq: def-btepsilon} we have 
	$$\epsilon_t\cdot \left(f^{\bepsilon}_i(x_t) - s_t\right) = \frac{\alpha}{2}.$$
	For $t = i+1$, based on our construction in \pref{eq: f-i+1} and our choice of $\btepsilon$ in \pref{eq: def-btepsilon}, we can calculate that 
	$$\epsilon_t\cdot \left(f^{\bepsilon}_i(x_t) - s_t\right) = \frac{\alpha}{2}.$$
	For $t > i+1$, since
	$$\epsilon_t\cdot (f^\bepsilon_i(x_t) - s_t) \ge \frac{\alpha}{2}$$
	according to \pref{eq: f-x-i}, hence according to our choice of $\btepsilon$ in \pref{eq: def-btepsilon} we have 
	$$\epsilon_t\cdot \left(f^{\bepsilon}_i(x_t) - s_t\right)\ge \frac{\alpha}{2}.$$
	Therefore, $f^{\bepsilon}_{i+1}$ satisfies \pref{eq: f-x-i} as well. By choosing $i = d$, we obtain that there exist $f^\bepsilon_d\in \calF$ such that 
	$$\epsilon_t\cdot (f^\bepsilon_d(x_t) - s_t) = \frac{\alpha}{2},$$
	which implies that $\{x_1, \ldots, x_d\}$ is shattered by $\calF$ at scale $\alpha$ under witness $s_1, \ldots, s_d$ under \pref{def: fixed-scale dim}.
\end{proof}

\subsection{Missing Proofs in \pref{sec:lower_bound_offset_nonsequential}}\label{sec: offset-lower-bound-2-app}
\begin{proof}[Proof of \pref{thm: offset-lower-bound-2}]
	$$\alpha_n = \frac{1}{2}\cdot \argmin_{\alpha > 0}\left\{\dnon(\calF, \alpha, \alpha/(20nC))\cdot \frac{16}{\alpha^2}\le n\right\}.$$
	Since $\dnon(\calF, \alpha, \alpha/(20nC)) = \tOmega\left(\alpha^{-p}\right)$ for every $\alpha\ge 0$, we have
	$$\dnon(\calF, \alpha_n, \alpha_n/(20nC))\wedge \frac{n\alpha_n^2}{4} = \tOmega\left(n^{\frac{p}{p+2}}\right),\qquad \forall n\in \mathbb{Z}_+.$$
	In the following, we will prove that for any positive integer $n$, we have 
	$$\sup_{\bmu, \bx} \EE[\calR_n(\calF, \bmu, \bx, \bepsilon)] = \Omega\left(\dnon(\calF, \alpha_n, \alpha_n/(20nC))\right).$$

	We fix $n$, and, for brevity, write
    $$\alpha = \alpha_n\quad \text{and}\quad d = \dnon(\calF, \alpha_n, \alpha_n/(20nC))\wedge n\alpha^2/4\le \dnon(\calF, \alpha_n, \alpha_n/(20nC)).$$
    Let $\tx_{1:d}\in \calX^d$ be a sequence shattered by $\calF$ in the sense of \pref{def: non-sequential-real}. That is, there exist $s_1 = (s_1[-1], s_1[1]), s_2 = (s_2[-1], s_2[1]), \ldots, s_d = (s_d[-1], s_d[1])\in [-1, 1]\times [-1, 1]$ such that for any $\btepsilon = (\tepsilon_{1:d})\in \{-1, 1\}^d$, there exists $f^{\btepsilon}\in \calF$ such that
	\begin{equation}\label{eq: nonsequential-eq-natarajan}
		\left|f^{\btepsilon}(\tx_t) - s_t[\tepsilon_t]\right| \leq \frac{\alpha}{20(nC)}\quad \text{and}\quad |s_t[1] - s_t[-1]| \ge \alpha\qquad \forall t\in [d].
	\end{equation} 
	Without loss of generality, we assume $s_t[1]\ge s_t[-1]$. We now define $\tmu_{1:d}\in [-1, 1]^d$ as $\tmu_t = (s_t[-1] + s_t[1]) / 2$ and  
	$$k_t = \left\lfloor \frac{1}{(s_t[1] - s_t[-1])^2}\right\rfloor \vee 1\le 4\alpha^{-2}$$
	for any $t\in [d]$. Then we have 
	$$\sum_{t=1}^d k_t \le n.$$
	
    Next, we construct $x_{1:n}$ and $\mu_{1:n}$ with the following block structure:
	\begin{align*} 
		(x_{1:n}) & = (\underbrace{\tx_1, \tx_1, \ldots, \tx_1}_{k_1 \text{ terms}}, \underbrace{\tx_2, \tx_2, \ldots, \tx_2}_{k_2\text{ terms}}, \ldots, \underbrace{\tx_{d-1}, \tx_{d-1}, \ldots, \tx_{d-1}}_{k_{d-1}\text{ terms}}, \underbrace{\tx_d, \tx_d, \ldots, \tx_d}_{n - k_1 - \ldots - k_{d-1}\text{ terms}}),\\
		(\mu_{1:n}) & = (\underbrace{\tmu_1, \tmu_1, \ldots, \tmu_1}_{k_1 \text{ terms}}, \underbrace{\tmu_2, \tmu_2, \ldots, \tmu_2}_{k_2 \text{ terms}}, \ldots, \underbrace{\tmu_{d-1}, \tmu_{d-1}, \ldots, \tmu_{d-1}}_{k_{d-1}\text{ terms}}, \underbrace{s_d[1], s_d[1], \ldots, s_d[1]}_{n - k_1 - \ldots - k_d \text{ terms}}).
	\end{align*}

	Next we write $\calR_n(\calF, \mu_{1:n}, x_{1:n}, \bepsilon)$ in terms of segments $1$ to $d$. For any $\bepsilon\in  \{0, 1\}^n$, we construct $\btepsilon\in \{-1, 1\}^{d}$: 
	$$\tepsilon_t = 2\mathbb{I}\bigg\{\sum_{j=1}^{k_t} \epsilon_{k_1 + \ldots + k_{t-1} + j} \ge 0\bigg\}-1,\qquad \forall t\in [d-1]\quad \text{and}\quad \tepsilon_d = 1.$$
	Recall that $f^{\btepsilon}\in \calF$ for any sequence $\btepsilon\in \{-1, 1\}^d$. Then, using $[t, j]$ to denote $k_1 + \ldots + k_t + j$, and defining the random variable $\hepsilon_t = \frac{1}{k_t} \sum_{j=1}^{k_t} \epsilon_{[t-1, j]}\in [-1, 1]$ for fixed $\bepsilon$, we obtain 
	\begin{align*} 
		&\hspace{-0.5cm} \calR_n(\calF, \mu_{1:n}, x_{1:n}, \bepsilon)\\
		& = \sup_{f\in \calF} \Bigg\{\sum_{t=1}^{d-1}
        \sum_{j=1}^{k_t} C\cdot  \epsilon_{[t-1, j]} (f(x_{[t-1, j]}) - \mu_{[t-1, j]}) - (f(x_{[t-1, j]}) - \mu_{[t-1, j]})^2\\
		&\qquad + \sum_{j=1}^{n - k_1 - \ldots - k_{d-1}} C\cdot  \epsilon_{[d-1, j]} (f(x_{[d-1, j]}) - \mu_{[d-1, j]}) - (f(x_{[d-1, j]}) - \mu_{[d-1, j]})^2\Bigg\}\\
		& \stackrel{(i)}{\ge} \sum_{t=1}^{d-1}\sum_{j=1}^{k_t} C\cdot\epsilon_{[t-1, j]} (f^{\btepsilon}(\tx_t) - \tmu_t) - (f^{\btepsilon}(\tx_t) - \tmu_t)^2 - n\cdot 2C\cdot \frac{\alpha}{nC}\\
		& \ge \sum_{t=1}^{d - 1} k_t\cdot \left(C\cdot\hepsilon_t(f^{\btepsilon}(\tx_t) - \tmu_t) - (f^{\btepsilon}(\tx_t) - \tmu_t)^2\right) - 4, \numberthis \label{eq: nonsequential-r-n-segment}
	\end{align*}
 	where $(i)$ uses the choice $f = f^{\btepsilon}\in \calF$, and also $|f^{\btepsilon}(\tx_d) - s_d[1]|\le \alpha/(20nC) < \alpha/(nC)$ since $\tepsilon_d = 1$, and also the fact that $\alpha\le 2$ due to \pref{eq: nonsequential-eq-natarajan}. According to \pref{eq: nonsequential-eq-natarajan}, we have for any $\btepsilon$, 
	$$|f^{\btepsilon}(\tx_t) - s_t[\tepsilon_t]|\le \frac{\alpha}{20}\quad\text{and}\quad s_t[\tepsilon_t] - \mu_t = \sign(\hepsilon_t)\cdot \frac{s_t[1] - s_t[-1]}{2}.$$
	\pref{eq: nonsequential-eq-natarajan} also gives that $s_t[1] - s_t[-1] \geq \alpha$, which implies
	$$\frac{9}{10}\cdot \frac{s_t[1] - s_t[-1]}{2}\le \sign(\hepsilon_t)(f^{\btepsilon}(\tx_t) - \tmu_t)\le \frac{11}{10}\cdot \frac{s_t[1] - s_t[-1]}{2}.$$
	Therefore we can lower bound \pref{eq: nonsequential-r-n-segment} by 
	$$\calR_n(\calF, \mu_{1:n}, x_{1:n}, \bepsilon) \ge \sum_{t=1}^{d-1} k_t\cdot \left(C |\hepsilon_t|\cdot \frac{9}{10}\cdot \frac{s_t[1] - s_t[-1]}{2} - \frac{121}{100}\cdot\left(\frac{s_t[1] - s_t[-1]}{2}\right)^2\right) - 4.$$ 
	Next, we can further lower bound 
	\begin{align*} 
		&\hspace{-0.5cm} \calR_n(\calF, \mu_{1:n}, x_{1:n}, \bepsilon)\\
		& \ge \sum_{t=1}^{d-1} k_t\cdot \left(C |\hepsilon_t|\cdot \frac{9}{10}\cdot \frac{s_t[1] - s_t[-1]}{2} - \frac{121}{100}\cdot\left(\frac{s_t[1] - s_t[-1]}{2}\right)^2\right) - 4\\
		& = \sum_{t=1}^{d-1} k_t\cdot \left(C \EE\left[|\hepsilon_t|\right]\cdot \frac{9}{10}\cdot \frac{s_t[1] - s_t[-1]}{2} - \frac{121}{100}\cdot\left(\frac{s_t[1] - s_t[-1]}{2}\right)^2\right) - 4\\
		& \ge \sum_{t=1}^{d-1} k_t\cdot \left(C \sqrt{\frac{1}{2k_t}}\cdot \frac{9}{10}\cdot \frac{s_t[1] - s_t[-1]}{2} - \frac{121}{100}\cdot\left(\frac{s_t[1] - s_t[-1]}{2}\right)^2\right) - 4 \numberthis \label{eq: nonsequential-equation-lower-decompose}
	\end{align*}
	where the last inequality uses the Khintchine inequality \cite{haagerup1981best}. According to our choice of $k_t$, 
	$$k_t = \left\lfloor \frac{1}{(s_t[1] - s_t[-1])^2}\right\rfloor \vee 1 \in \left[\frac{1}{2(s_t[1] - s_t[-1])^2}, \frac{4}{(s_t[1] - s_t[-1])^2}\right], $$
	which implies that $1/\sqrt{2}\le \sqrt{k_t} (s_t[1] - s_t[-1])\le 2$. Therefore, since $C\ge 2$, we have 
	\begin{align*}
		\text{RHS of }\pref{eq: nonsequential-equation-lower-decompose} 
		& \ge \sum_{t=1}^{d-1} \left(\frac{9C}{20\sqrt{2}} - \frac{121}{200}\right)\cdot \sqrt{k_t}\cdot |s_t[1] - s_t[-1]| - 4\ge \frac{d-1}{50} - 4.
	\end{align*}
	Therefore, we obtain that 
	$$\sup_{\mu_{1:n}, x_{1:n}} \EE[\calR_n(\calF, \mu_{1:n}, x_{1:n}, \bepsilon)]\ge \frac{d-1}{50} - 4 = \tOmega\left(n^{\frac{p}{p+2}}\right).$$
\end{proof}

\begin{proof}[Proof of \pref{lem:transductive}]
	We first transform $\calV_n(\calF)$ into the following dual form
	\begin{align*}
		\calV_n(\calF) & = \sup_{x_{1:n}\in \calX^n} \left\{\inf_{\hy_t} \sup_{p_t\in \Delta([-2, 2])}\EE_{y_t\sim p_t}\right\}_{t=1}^n\left[\sum_{t=1}^n \left(\hy_t - y_t\right)^2 - \inf_{f\in \calF} \sum_{t=1}^n\left(f(x_t) - y_t\right)^2\right]\\
		& = \sup_{x_{1:n}\in \calX^n} \left\{\sup_{p_t\in \Delta([-2, 2])}\EE_{y_t\sim p_t}\right\}_{t=1}^n \left[\sum_{t=1}^n \left(\EE[y_t\mid y_{1:t-1}] - y_t\right)^2 - \inf_{f\in \calF} \sum_{t=1}^n\left(f(x_t) - y_t\right)^2\right]\\
		& = \sup_{x_{1:n}}\sup_{\bp\in \Delta([-2, 2]^{n})}\EE_{\by\sim \bp} \left[\sup_{f\in \calF}\sum_{t=1}^n 2(y_t - \mu_t(\by))(f(x_t) - \mu_t(\by)) - (f(x_t) - \mu_t(\by))^2\right]
	\end{align*}
	where we use $\mu_t(\by)$ to denote the expectation of $y_t$ conditioned on $y_{1:t-1}$, i.e. $\mu_t(\by) = \EE[y_t\mid y_{1:t-1}]$. Taking $\bp = p_1\otimes p_2\otimes\ldots \otimes p_n$ where $p_t\in \Delta([-2, 2])$ for each $t\in [n]$ in the above equation, we obtain that
	$$\calV_n(\calF)\ge \sup_{x_{1:n}}\sup_{p_{1:n}\in \Delta([-2, 2]} \EE_{y_t\sim p_t}\left[\sup_{f\in \calF} \sum_{t=1}^n 2(y_t - \mu_t)(f(x_t) - \mu_t) - (f(x_t) - \mu_t)^2\right],$$
	where $\mu_t\coloneqq \EE[y_t]$ denotes the expectation of distribution $p_t$.

	We fix $\mu_1, \ldots, \mu_n\in [-1, 1]$, and choose distributions $p_t = \text{Unif}(\{\mu_t - 1, \mu_t + 1\})\subseteq \Delta([-2, 2])$ for all $t\in [n]$. Then we obtain that 
	$$\calV_n(\calF)\ge \sup_{x_{1:n}}\sup_{\mu_{1:n}\in [-1, 1]} \EE_{\epsilon_t}\left[\sup_{f\in \calF} \sum_{t=1}^n 2 \epsilon_t(f(x_t) - \mu_t) - (f(x_t) - \mu_t)^2\right],$$ 
	where $\epsilon_{1:n}\simiid \unif(\{-1, 1\})$.
\end{proof}

\begin{proof}[Proof of \pref{corr: transductive}]
    Since $\sup_{x_{1:n}}\log \calNnon_\infty(\calF, x_{1:n}, \alpha) = \tOmega\left(\alpha^{-p}\right)$, according to \pref{prop: non-covering-fat} we have 
    $$\dnon(\calF, \alpha, \alpha/(40n)) = \tOmega\left(\alpha^{-p}\right).$$
    Next, we call \pref{lem:transductive}, and \pref{thm: offset-lower-bound-2} with $C = 2$. Then we obtain
    $$\calV_n(\calF) = \tOmega\left(n^{\frac{p}{p+2}}\right).$$
\end{proof}

\section{Missing Proofs in \pref{sec:sequential}}\label{sec: lower-bound-proof}
\subsection{Proof of \pref{lem: shattering-finite}}
\begin{proof}[Proof of Lemma \ref{lem: shattering-finite}]
	We first define function
	$$g_M(n, d) = \sum_{i=0}^d \binom{n}{i}\cdot \left(M - 1\right)^i,$$
	which satisfies (see \citep{rakhlin2010online})
	\begin{equation}\label{eq: equation-g-beta}
		g_M(n, d) = g_M(n-1, d) + \left(M - 1\right)\cdot g_M(n-1, d-1).
	\end{equation}

	We will prove the present Lemma by induction with the following induction statement:\\

    \begin{minipage}[c]{0.1\linewidth}
	\hspace{1mm}${\mathfrak G}(d, n)$:\\
	\end{minipage}\hspace{2mm}
	\begin{minipage}[b]{0.85\linewidth}For any function class $\calF\subseteq  \{f:\cX\to[M]\}$ with $\dseq(\calF, \alpha)\le d$, and any depth-$n$ $\calX$-valued tree $\bx$, $\calNseq_\infty(\calF, \bx, \alpha)\le g_M(d, n)$.
	\end{minipage}

	\paragraph{Base: } There are two base cases to consider: $n \le d$ and $d = 0$.

	When $n\le d$, we let 
	$$\calV = \left\{\bv[i_1, i_2, \ldots, i_n]: i_1, \ldots, i_n\in [M]\right\},$$
	where $\bv[i_1, \ldots, i_n]$ denotes the tree with\ values $i_t$ at depth $t$ along any path. Then it is easy to see that for any $f\in \calF$, depth-$n$ $\calX$-valued tree $\bx$, and any path $\bepsilon\in \{-1, 1\}^n$, there exists some $\bv\in \calV$ such that $f(x_t(\bepsilon)) = v_t(\bepsilon)$ for all $t\in [n]$. Hence $\calV$ is an exact (that is, $\alpha=0$) cover of $\calF$ on $\bx$; hence, $\calV$ is also an $\alpha$-sequential covering as well. Thus, we have
	$$\calNseq_\infty(\calF, \bx, \alpha)\le |\calV| = M^n = \sum_{i=0}^d \binom{n}{i}\cdot \left(M - 1\right)^i = g_M(n, d).$$

	When $d = 0$, there is no depth-$1$ $\calX$-valued tree which is shattered by $\calF$ at scale $\alpha$ under the metric $\frakc$. This implies for any $x\in \calX$ and functions $f, f'\in \calF$, we always have $\frakc(f(x), f'(x))\le \alpha$. Indeed, otherwise we can construct depth-$1$ tree $\bx$ with $x_1 = x$ and depth-$1$ tree $\bs$ with $s_1 = (f(x), f'(x)$), and $\bx$ is shattered by $\calF$. We let $f_0\in \calF$ to be an arbitrary function in $\calF$. For any depth-$n$ $\calX$-valued $\bx$, we construct depth-$n$ $[-1, 1]$-valued tree $\bv$ which takes the value $f_0(x_t(\bepsilon))$ at depth $t$ along any path $\bepsilon$. Then for any $f\in \calF$ and path $\bepsilon\in \{-1, 1\}$, we always have $\frakc(f(x_t(\bepsilon)), v_t(\bepsilon)) =  \frakc(f(x_t(\bepsilon)), f_0(x_t(\bepsilon)))\le \alpha$. Hence $\calV$ is an $\alpha$-sequential covering of $\calF$ on $\bx$, and it satisfies $|\calV| = 1 = g_M(n, 0)$.

	\paragraph{Induction: } Suppose the induction hypotheses $\mathfrak{G}(n-1, d-1)$ and $\mathfrak{G}(n-1, d)$ both hold. We will prove induction statement $\mathfrak{G}(n, d)$. For fixed function class $\calF$ with $\dseq(\calF, \alpha) = d$ and depth-$n$ $\calX$-valued tree $\bx$, we will construct a $\alpha$-sequential covering to of $\calF$ on $\bx$ whose size is no more than $g_M(n, d)$. Suppose the root of tree $\bx$ is $x_1\in \calX$, the left subtree of $x_1$ is denoted as $\bx^l$, and the right subtree of $x_1$ is denoted as $\bx^r$. We partition the function class $\calF$ as:
	$$\calF = \calF_1\cup\calF_2\cup\ldots\cup \calF_M\quad \text{where}\quad \calF_k = \{f\in \calF: f(x_1) = k\},\quad \forall 1\le k\le M.$$
	Then we have $\dseq(\calF_k, \alpha)\le \dseq(\calF, \alpha) = d$ for all $k\in [M]$.  We let $\calK = \{k\in [M]: \dseq(\calF_k, \alpha) = d\}$. Then for any $a, b\in \calK$ and $a < b$, there exist depth-$d$ $\calX$-valued trees $\bx^a$ and $\bx^b$, and also depth-$d$ $[M]\times [M]$-valued trees $\bs^a$ and $\bs^b$ such that for any $\bepsilon\in \{-1, 1\}^d$, there exists $f_a^\bepsilon\in \calF_a$ and $f_b^\bepsilon\in \calF_b$ such that for any $t\in [d]$,
	\begin{align*} 
		f_a^\bepsilon(x_t^a(\bepsilon)) = s_t^a(\bepsilon)[\epsilon_t]\quad & \text{and}\quad f_b^\bepsilon(x_t^b(\bepsilon)) = s_t^b(\bepsilon)[\epsilon_t],\\
		\frakc(s_t^a(\bepsilon)[-1], s_t^a(\bepsilon)[1]) \ge \alpha\quad & \text{and}\quad \frakc(s_t^b(\bepsilon)[-1], s_t^b(\bepsilon)[1]) \ge \alpha.
	\end{align*}
	For the sake of a contradiction, suppose it holds that $\frakc(a, b) \ge \alpha$. Then we can construct a depth-$(d+1)$ $\calX$-valued tree $\bx$ with root $x_1$, left subtree of the root being $\bx^a$, and right subtree of the root being $\bx^b$, and also a depth-$(d+1)$ $[M]\times [M]$-valued tree $\bs$ with root $(a, b)\in [M]\times [M]$, left subtree of the root being $\bs^a$, and right subtree of the root being $\bs^b$. Then we can verify that for any $\bepsilon\in \{-1, 1\}^{d+1}$, and any $t\in [d+1]$, we have $s_t(\bepsilon)[-1] < s_t(\bepsilon)[1]$, and $\frakc(s_t(\bepsilon)[-1], s_t(\bepsilon)[1]) \ge \alpha$. Further, we let $\bepsilon' = (\epsilon_2, \epsilon_3, \ldots, \epsilon_{d+1})\in \{-1, 1\}^d$, and if $\epsilon_1 = -1$, then letting $f^\bepsilon = f_a^{\bepsilon'}$ we can verify that $f^\bepsilon(x_t(\bepsilon)) = s_t(\bepsilon)[\epsilon_t]$ for any $t\in [d+1]$, and if $\epsilon_1 = 1$, then letting $f^\bepsilon = f_b^{\bepsilon'}$ we can verify that $f^\bepsilon(x_t(\bepsilon)) = s_t(\bepsilon)[\epsilon_t]$ for any $t\in [d+1]$. Hence, depth-$(k+1)$ tree $\bx$ is shattered by $\calF$, which leads to contradiction. Therefore, we have 
	\begin{equation}\label{eq: a-b}
		\frakc(a, b) < \alpha,\qquad \forall a, b\in \calK.
	\end{equation}
	
	Next, for any $k\in [M]$ with $\dseq(\calF_k, \alpha)\le d-1$, according to the induction hypothesis $\mathfrak{G}(n-1, d-1)$, there exists a sequential cover $\calV_k^l$ of size $g_M(n-1, d-1)$ for $\cF$ on the depth-$(n-1)$ $\calX$-valued tree $\bx^l$, and also a sequential cover $\calV_k^r$ of size $g_M(n-1, d-1)$ for $\cF$ on the depth-$(n-1)$ $\calX$-valued tree $\bx^r$. We then combine the elements in $\calV_k^l$ and $\calV_k^r$ into a set $\calV_k$ of depth-$n$ $[M]$-valued trees by a joining process as follows. We let $v_1 = k\in [M]$, and according to the construction of $\calF_k$ we have for any $f\in \calF$ that $f(x_1) = v_1$, and thus $\frakc(f(x_1), v_1)\le \alpha$. For $\bv^l\in \calV_k^l$ and $\bv^r\in \calV_k^r$, we define depth-$n$ $[M]$-valued tree $\bv[\bv^l, \bv^r]$ as: for any path $\bepsilon\in \{-1, 1\}^n$, we let $\bepsilon' = (\epsilon_{2:n})\in \{-1, 1\}^{n-1}$, and let $v_1[\bv^l, \bv^r](\bepsilon) = v_1$. If $\epsilon_1 = -1$, then let $v_t[\bv^l, \bv^r](\bepsilon) = v_{t-1}^l(\bepsilon')$, and if $\epsilon_1 = 1$, then let $v_t[\bv^l, \bv^r](\bepsilon) = v_{t-1}^r(\bepsilon')$. We construct $\calV_k = \{\bv[\bv^l, \bv^r]\}$ with $|\calV_k|\le \max\{|\calV_k^l|, |\calV_k^r|\}$ to make sure that every element in $\calV_k^l$ and $\calV_k^r$ appears at least once in the construction of $\calV_k$. Next, we will argue that $\calV_k$ is an $\alpha$-sequential cover of $\calF_k$ on $\bx$. This is easy to see by construction: for any $f\in \calF_k$ and $\bepsilon\in \{-1, 1\}^n$, if $\epsilon_1 = -1$, then since $\calV_k^l$ is a $\alpha$-sequential cover of $\calF_k$, there exists $\bv^l\in \calV_k^l$ such that for any $2\le t\le n$, $\frakc(f(x_t(\bepsilon)), v_t^l(\bepsilon))\le \alpha$. Suppose $\bv = \bv[\bv^l, \bv^r]\in \calV_k$ for some $\bv^r\in \calV_k^r$, and we also have $\frakc(f(x_1(\bepsilon)), v_1(\bepsilon))\le \alpha$ according to the construction of $\calF_k$. Hence, for any $t\in [n]$, we always $\frakc(f(x_t(\bepsilon)), v_t(\bepsilon))\le \alpha$. Therefore, $\calV_k$ is a cover of $\calF_k$ on $\bx$. Further by induction hypothesis we have $\max\{|\calV_k^l|, |\calV_k^r|\}\le g_M(n-1, d-1)$. Hence $|\calV_k|\le g_M(n-1, d-1)$.

	If $\calK = \emptyset$, then by letting $\calV = \cup_{k\in [M]} \calV_k$, $\calV$ is an $\alpha$-sequential cover of $\calF$ on $\bx$, and also 
	$$|\calV| \le M\cdot g_M(n-1, d-1)\le g_M(n-1, d) + (M-1)g_M(n-1, d-1) = g_M(n, d),$$
	where the inequality follows from the fact that $g_M(n-1, d-1)\le g_M(n-1, d)$ for any $n, d$, and the last equation follows from \pref{eq: equation-g-beta}.
	
	Next, we consider cases where $|\calK|\ge 1$. Consider the function class $\calF' = \cup_{k\in \calK} \calF_k\subseteq \calF$. We have $\dseq(\calF', \alpha)\le \dseq(\calF, \alpha) = d$. According to the induction hypothesis $\mathfrak{G}(n-1, d)$, there exists a sequential cover $\calV^l$ of size $g_M(n-1, d)$ for the depth-$(n-1)$ $\calX$-valued tree $\bx^l$, and also a sequential cover $\calV^l$ of size $g_M(n-1, d)$ for the depth-$(n-1)$ $\calX$-valued tree $\bx^l$. As before, we combine the elements in $\calV^l$ and $\calV^r$ into a set $\calV'$ of depth-$n$ $[M]$-valued trees. We let $v_1 = f(x_1)\in [M]$ for some $f\in \calF'$, chosen arbitrarily. Then, according to the construction of $\calF'$, we have for any other $g\in \calF'$, $\frakc(g(x_1), v_1)\le \alpha$. For $\bv^l\in \calV^l$ and $\bv^r\in \calV^r$, we define depth-$n$ $[M]$-valued tree $\bv[\bv^l, \bv^r]$ by joining them with $v_1$ at the root as before: for any path $\bepsilon\in \{-1, 1\}^n$, we let $\bepsilon' = (\epsilon_{2:n})\in \{-1, 1\}^{n-1}$, and let $v_1[\bv^l, \bv^r](\bepsilon) = v_1$. If $\epsilon_1 = -1$, then let $v_t[\bv^l, \bv^r](\bepsilon) = v_{t-1}^l(\bepsilon')$, and if $\epsilon_1 = 1$, then let $v_t[\bv^l, \bv^r](\bepsilon) = v_{t-1}^r(\bepsilon')$. We construct $\calV' = \{\bv[\bv^l, \bv^r]\}$ with $|\calV'|\le \max\{|\calV^l|, |\calV^r|\}$ to make sure that every element in $\calV^l$ and $\calV^r$ appears at least once in the construction of $\calV'$. Next, we will argue that $\calV'$ is a $\alpha$-sequential cover of $\calF'$ on $\bx$. For any $f\in \calF'$ and $\bepsilon\in \{-1, 1\}^n$, if $\epsilon_1 = -1$, then since $\calV^l$ is a $\alpha$-sequential cover of $\calF'$ on the tree $\bx^l$, there exists $\bv^l\in \calV^l$ such that for any $2\le t\le n$, $\frakc(f(x_t(\bepsilon)), v_t^l(\bepsilon))\le \alpha$. Suppose $\bv = \bv[\bv^l, \bv^r]\in \calV'$ for some $\bv^r\in \calV^r$ (according to the construction of $\calV'$ such $\bv$ exists). Then since $f\in \calF'$, according to \pref{eq: a-b} we have $\frakc(f(x_1(\bepsilon)), v_1(\bepsilon))\le \alpha$. Hence for any $t\in [n]$, we always $\frakc(f(x_t(\bepsilon)), v_t(\bepsilon))\le \alpha$. Similarly, if $\epsilon_1 = 1$, there also exists some $\bv\in \calV'$ such that $\frakc(f(x_t(\bepsilon)), v_t(\bepsilon))\le \alpha$ for any $t\in [n]$. Therefore, $\calV'$ is an $\alpha$-sequential cover of $\calF'$. Further by induction hypothesis we have $\max\{|\calV^l|, |\calV^r|\}\le g_M(n-1, d)$. Hence $|\calV'|\le g_M(n-1, d)$. We now let $\calV = \calV'\cup\left(\cup_{k\not\in \calK} \calV_k\right)$, and after noticing that $|[M]\backslash \calK|\le (M - 1)$, we obtain
	$$|\calV| \le (M-1)\cdot g_M(n-1, d-1) + g_M(n-1, d) = g_M(n, d),$$
	where the last equation follows from \pref{eq: equation-g-beta}. This finishes the proof of the induction statement $\mathfrak{G}(n, d)$.

	We conclude that, by induction, we have for any depth-$n$ $\calX$-valued tree $\bx$, 
	$$\calNseq_\infty(\calF, \bx, \alpha)\le g_M(n, \dseq(\calF, \alpha)) = \sum_{i=0}^{\dseq(\calF, \alpha)} \binom{n}{i}\cdot \left(M - 1\right)^i \stackrel{(i)}{\le} \left(\frac{enM}{\dseq(\calF, \alpha)}\right)^{\dseq(\calF, \alpha)}\le \left(enM\right)^{\dseq(\calF, \alpha)},$$
	where $(i)$ follows e.g. from \cite[Theorem 7]{rakhlin2015sequentialcomplexity}. 
\end{proof}

\begin{proof}[Proof of \pref{prop: seq-covering-fat}]
	We define distance $\frakc': [M]\times [M]\to \RR_+\cup\{0\}$:
	\begin{equation}\label{eq: def-cprime-c}
		\frakc'(a, b) = \frakc(u_a, u_b).
	\end{equation}
	For any $f\in \calF$, since $f$ maps $\calX$ into $[-1, 1]$, there exists $\bar{f}: \calX\to [M]$ such that for any $x\in \calX$, $\frakc(f(x), u_{\barf(x)})\le \beta$. We define function class $\barF = \{\barf: f\in \calF\}\subseteq \{\calX\to [M]\}$. We use $\dseq_{\frakc'}(\barF, \alpha)$ to denote the sequential gapped 
	dimension of integer-valued function class $\barF$ at scale $\alpha$ under distance $\frakc'$, where the dimension is defined in \pref{def: shattering-finite}, and for simplicity we let $d = \dseq_{\frakc'}(\barF, \alpha)$. Suppose $\bx$ is a depth-$d$ tree shattered by $\calF$. Then there exists a depth-$d$ $([M]\times [M])$-valued tree $\barbs$ such that for any $\bepsilon\in \{-1, 1\}^d$ and $t\in [d]$, $\frakc'(\bars_t(\bepsilon)[-1], \bars_t(\bepsilon)[1]) \ge \alpha$, and also for any $\bepsilon\in \{-1, 1\}^d$, there exists $\barf^{\bepsilon}\in \barF$ such that $\barf^\bepsilon(x_t(\bepsilon)) = \bars_t(\bepsilon)[\epsilon_t]$ for any $t\in [d]$. Notice from the definition of $\calF$ there exists some $f^\bepsilon\in \calF$ such that $\barf^\bepsilon = \overline{(f^\bepsilon)}$.

	We next verify that tree $\bx$ is also shattered by $\calF$ at scale $(\alpha, \beta)$, according to the definition \pref{def: seq-real-dim}. We construct depth-$d$ $([-1, 1]\times [-1, 1])$-tree $\bs$ according to $\barbs$ as follows:
	$$s_t(\bepsilon)[-1] = u_{\bars_t(\bepsilon)[-1]}\quad \text{and}\quad s_t(\bepsilon)[1] = u_{\bars_t(\bepsilon)[1]}.$$ 
	Then according to the definition of $\frakc'$ in \pref{eq: def-cprime-c}, we have for any $\bepsilon\in \{-1, 1\}^d$ and $t\in [n]$,
	$$\frakc(s_t(\bepsilon)[-1], s_t(\bepsilon)[1]) = \frakc'(\bars_t(\bepsilon)[-1], \bars_t(\bepsilon)[1]) \ge \alpha.$$
	Additionally, for any $\bepsilon\in \{-1, 1\}^d$, there exists some $f^\bepsilon\in \calF$ such that $\overline{(f^\bepsilon)}(x_t(\bepsilon)) = \bars_t(\bepsilon)[\epsilon_t]$, which implies,
	$$\frakc(f(x_t(\bepsilon)), s_t(\bepsilon)[\epsilon_t]) = \frakc(f(x_t(\bepsilon)), u_{\bars_t(\bepsilon)[\epsilon_t]}) = \frakc(f(x_t(\bepsilon)), u_{\overline{(f^\bepsilon)}(x_t(\bepsilon))})\le \beta,\quad \forall t\in [d],$$
	where the last inequality follows from the definiton of $\bar{f}$. Therefore, $\bx$ is shattered by $\calF$ at scale $(\alpha, \beta)$, and according to \pref{def: seq-real-dim} we have $\dseq(\calF, \alpha, \beta)\ge d$. 

	Next we will upper bound the sequential covering number of $\calF$ in terms of $d$. For a fixed depth-$n$ $\calX$-valued tree $\bx$, according to \pref{lem: shattering-finite}, there exists a sequential covering $\barV$ of $\barF$ with size no more than $(neM)^d$. Hence for any $f\in \calF$ and $\bepsilon\in \{-1, 1\}^d$, since $\bar{f}\in \barF$, there exists some $\barbv\in \barV$ which satisfies
	$$\frakc'(\barf(x_t(\bepsilon)), \barv_t(\bepsilon))\le \alpha\qquad \forall t\in [n],$$
	which implies that 
	$$\frakc(f(x_t(\bepsilon)), u_{\barv_t(\bepsilon)})\le \frakc(f(x_t(\bepsilon)), u_{\bar{f}(x_t(\bepsilon))}) + \frakc(u_{\bar{f}(x_t(\bepsilon))}, u_{\barv_t(\bepsilon)})\le \beta + \alpha\qquad\forall t\in [n],$$
	where the first inequality uses the triangle inequality of $\frakc$, and the second inequality uses the definition of function $\bar{f}$. For every $\barbv\in \barV$, we construct depth-$d$ $\RR$-valued tree $\bv_{\barbv}$ where for any $\bepsilon\in \{-1, 1\}^d$ and $t\in [n]$, $(v_{\barbv})_t(\bepsilon) = u_{\barv_t(\bepsilon)}$, for every $\barbv\in \barV$. And we further let $\calV = \{\bv_{\barbv}: \barbv\in \barV\}$. Then $\calV$ is an $(\alpha + \beta)$-cover of $\calF$ on tree $\bx$, which implies 
	$$\calNseq_\infty(\calF, \bx, \alpha + \beta)\le |\calV| = |\barV|\le (neM)^d.$$
\end{proof}

\subsection{Missing Proofs in \pref{sec: f-d}}\label{sec: f-d-proof}
\begin{proof}[Proof of \pref{prop: d-less-f}]
	We only need to prove that if a depth-$d$ $\calX$-valued tree $\bx$ is shattered by $\calF$ at scale $(\alpha, \beta)$ according to \pref{def: seq-real-dim}, then $\bx$ is shattered by $\calF$ at scale $\alpha - 2\beta$ according to \pref{def: sequential-fat}.

	If $\bx$ is shattered by $\calF$ at scale $(\alpha, \beta)$ according to \pref{def: seq-real-dim}, then there exists a depth-$d$ $([0, 1]\times [0, 1])$-valued tree $\bs$ such that for any $\bepsilon\in \{-1, 1\}^n$, there exists a function $f^{\bepsilon}\in \calF$ such that for any $t\in [d]$, we have
	\begin{equation}\label{eq: f-s-inequaltiy}
		|f^{\bepsilon}(x_t(\bepsilon)) - s_t(\bepsilon)[\epsilon_t]|\le \beta\quad \text{and}\quad |s_t(\bepsilon)[1] - s_t(\bepsilon)[-1]| \ge \alpha.
	\end{equation}
	Without loss of generality we assume $s_t(\bepsilon)[1] > s_t(\bepsilon)[-1]$ for any $\bepsilon$ and $t\in [d]$. We define depth-$d$ $[0, 1]$-valued tree $\bv$ as follows: for any $\bepsilon\in \{-1, 1\}^d$ and $t\in [d]$, let 
	$$v_t(\bepsilon) = \frac{s_t(\bepsilon)[-1] + s_t(\bepsilon)[1]}{2}.$$
	Since $\alpha > 2\beta$, according to \pref{eq: f-s-inequaltiy} we have for any $\bepsilon\in \{-1, 1\}^n$, 
	$$\epsilon_t\cdot \left(f^{\bepsilon}(x_t(\bepsilon)) - v_t(\bepsilon)\right)\ge \frac{\alpha}{2} - \beta.$$
	Therefore, $\bx$ is shattered by $\calF$ with $\bv$ being its witness at scale $\alpha - 2\beta$, according to \pref{def: sequential-fat}.
\end{proof}

\begin{proof}[Proof of \pref{prop: f-less-d}]
    We first show that if depth-$d$ $\calX$-valued tree $\bx$ is shattered by function class $\calF$ at scale $\alpha$, according to \pref{def: sequential-fat}, then we have 
	\begin{equation}\label{eq: N-lower-bound}
		\calNseq_\infty(\calF, \bx, \alpha/3)\ge 2^d.
	\end{equation}
	We use depth-$d$ $[-1, 1]$-valued tree $\bs$ to denote the witness of shattering of $\bx$ via $\calF$ at scale $\alpha$. According to \pref{def: sequential-fat}, for any $\bepsilon\in \{-1, 1\}^d$, there exists some function $f^\bepsilon\in \calF$ such that 
	\begin{equation}\label{eq: f-s-shattering} 
		\epsilon_t\cdot \left(f^\bepsilon(x_t(\bepsilon)) - s_t(\bepsilon)\right) \ge \frac{\alpha}{2}
	\end{equation}
	Next we will prove \pref{eq: N-lower-bound} via contradiction. Suppose there exists an $\ell_\infty$-sequential covering $\calV$ at scale $\alpha/2$ of size less than $2^d$. Then for any $\bepsilon\in \{-1, 1\}^d$ and function $f\in \calF$, there exists some tree $\bv[f, \bepsilon]\in \calV$ such that 
	\begin{equation}\label{eq: f-v-covering}
		|f(x_t(\bepsilon)) - v_t(\bepsilon)|\le \frac{\alpha}{3},\qquad \forall t\in [d].
	\end{equation}
	Since $|\calV|\le 2^{d} - 1$, according to the pigeonhole principle there exists different $\bepsilon = (\epsilon_{1:d}), \bepsilon' = (\epsilon'_{1:d})\in \{-1, 1\}^d$ such that $\bv[f^\bepsilon, \bepsilon] = \bv[f^{\bepsilon'}, \bepsilon']$. We let 
	$$\bv[f^\bepsilon, \bepsilon] = \bv[f^{\bepsilon'}, \bepsilon'] = \bv$$
	We then choose $r$ to be smallest nonnegative integer such that $\epsilon_r\neq \epsilon'_r$. Since $\bepsilon\neq \bepsilon'$ we have $1\le r\le d$. Then we have $\epsilon_{1:r-1} = \epsilon'_{1:r-1}$, which implies that 
	\begin{equation}\label{eq: x-v-s-r}
		x_r(\bepsilon) = x_r(\bepsilon'), \quad v_r(\bepsilon) = v_r(\bepsilon')\quad \text{and}\quad s_r(\bepsilon) = s_r(\bepsilon').
	\end{equation}
	Therefore, we obtain that 
	\begin{align*} 
		|f^{\bepsilon}(x_r(\bepsilon)) - f^{\bepsilon'}(x_r(\bepsilon))| & = |f^{\bepsilon}(x_r(\bepsilon)) - f^{\bepsilon'}(x_r(\bepsilon'))|\\
		& \le |f^{\bepsilon}(x_r(\bepsilon)) - v_r(\bepsilon)| + |v_r(\bepsilon) - v_r(\bepsilon')| + |v_r(\bepsilon') - f^{\bepsilon'}(x_r(\bepsilon'))|\le \frac{\alpha}{3} + \frac{\alpha}{3} < \alpha, \numberthis \label{eq: f-epsilon-prime-bound}
	\end{align*}
	where the second inequality uses \pref{eq: f-v-covering} and the fact that $\bv[f^\bepsilon, \bepsilon] = \bv[f^{\bepsilon'}, \bepsilon'] = \bv$. Next according to \pref{eq: f-s-shattering}, we have 
	$$\epsilon_r\cdot \left(f^\bepsilon(x_r(\bepsilon)) - s_r(\bepsilon)\right) \ge \frac{\alpha}{2}\quad \text{and}\quad \epsilon_r'\cdot \left(f^{\bepsilon'}(x_r(\bepsilon')) - s_r(\bepsilon')\right) \ge \frac{\alpha}{2}.$$
	According to the definition of $r$ we have $\epsilon_r\neq \epsilon_r'$. Again using \pref{eq: x-v-s-r}, we obtain that
	$$|f^\bepsilon(x_t(\bepsilon)) - f^{\bepsilon'}(x_t(\bepsilon))| \ge \frac{\alpha}{2} + \frac{\alpha}{2} = \alpha,$$
	which contradicts \pref{eq: f-epsilon-prime-bound}. Therefore, we proved \pref{eq: N-lower-bound}.

	Next, according to \pref{prop: seq-covering-fat} with $\frakc(a, b) = |a - b|$, and tree $\bx$ being the depth-$\fat(\calF, 3(\alpha + \beta))$ $\calX$-valued tree shattered by $\calF$, we obtain that 
	$$\calNseq_\infty(\calF, \bx, \alpha+\beta)\le \left(\frac{2e\cdot \fat(\calF, 3(\alpha + \beta))}{\beta}\right)^{\dseq(\calF, \alpha, \beta)}.$$
	According to \pref{eq: N-lower-bound} we further have 
	$$\calNseq_\infty(\calF, \bx, \alpha + \beta)\ge 2^{\fat(\calF, 3\alpha + 3\beta)}.$$
	Therefore, we conclude that 
	$$\log 2\cdot \fat(\calF, 3(\alpha + \beta))\le \dseq(\calF, \alpha, \beta)\cdot \log\left(\frac{2e\cdot \fat(\calF, 3(\alpha + \beta))}{\beta}\right),$$
	which implies that 
        \begin{equation}\label{eq: seq-sfat-seq}
            \fat(\calF, 3(\alpha + \beta))\le 2\dseq(\calF, \alpha, \beta)\cdot \log\left(\frac{6\fat(\calF, 3(\alpha + \beta))}{\beta}\right).
        \end{equation}
        Additionally, since $\log x\le \sqrt{x}$ holds for any $x > 0$,
        $$\fat(\calF, 3(\alpha + \beta))\le 2\dseq(\calF, \alpha, \beta)\cdot \sqrt{\frac{6\fat(\calF, 3(\alpha + \beta))}{\beta}},$$
        which implies 
        $$\fat(\calF, 3(\alpha + \beta))\le \frac{24\dseq(\calF, \alpha, \beta)^2}{\beta}$$
        Bringing this back to \pref{eq: seq-sfat-seq}, we obtain that
	$$\fat(\calF, 3(\alpha + \beta))\le 4\dseq(\calF, \alpha, \beta)\cdot \log\left(\frac{12\dseq(\calF, \alpha, \beta)}{\beta}\right).$$
\end{proof}

\begin{proof}[Proof of \pref{prop: log-necessary}]
	We first show that for any $0 < 2\beta < \alpha < 1$, $\dseq(\calF, \alpha, \beta) = 1$. It is easy to see that the depth-$1$ $\calX$-valued tree $\bx$ with $x_1 = x$ is shattered by $\calF$ at scale $(\alpha, \beta)$, according to \pref{def: seq-real-dim}, hence $\dseq(\calF, \alpha, \beta)\ge 1$. Next we show that $\dseq(\calF, \alpha, \beta)\le 1$. Suppose there is a depth-$2$ $\calX$-valued tree $\bx$ shattered by $\calF$, then all nodes equal to $x$ whatever depth and path. We let $\bs$ to be the depth-$2$ $([-1, 1]\times [-1, 1])$-tree which is the witness of the shattering. Then for any $\bepsilon = (\epsilon_1, \epsilon_2)\in \{-1, 1\}^2$, there exists functions $f^{\bepsilon}\in \calF$ such that $|f^\bepsilon(x) - s_1(\bepsilon)[\epsilon_1]|\le \beta$ and $|f^\bepsilon(x) - s_2(\bepsilon)[\epsilon_2]|\le \beta$. Therefore, we have $|s_1(\bepsilon)[\epsilon_1] - s_2(\bepsilon)[\epsilon_2]|\le 2\beta$ for any $\bepsilon\in \{-1, 1\}^2$. We choose $\bepsilon = (1, 1)$ and $\bepsilon' = (1, -1)$, then we have 
	$$s_1(\bepsilon)[\epsilon_1] = s_1(\bepsilon')[\epsilon'_1]\quad \text{and}\quad |s_2(\bepsilon)[\epsilon_2] - s_2(\bepsilon')[\epsilon_2']|\ge \alpha.$$
	When $\alpha > 2\beta$, the above inequality cannot hold. Hence we have $\dseq(\calF, \alpha, \beta) = 1$.

	Next, we show that $\fat(\calF, \alpha)\ge \log(1/\alpha)$. We let $d = \lfloor \log_2(1/\alpha)\rfloor$, and for depth-$d$ path $\bepsilon\in \{-1, 1\}^d$, we define the shattered tree $\bx$ such that $x_t(\bepsilon) = x$ for any $\bepsilon\in \{-1, 1\}^d$ and $t\in [d]$, and the witness tree $\bv$ as:
	$$v_t(\bepsilon) = \sum_{i=1}^{t-1} \epsilon_i\cdot 2^{d-i}\cdot \alpha,$$
	and we choose function $f^\bepsilon$ as 
	$$f^\bepsilon(x) = \sum_{i=1}^d \epsilon_i\cdot 2^{d-i}\cdot \alpha.$$
	Since $d = \lfloor \log_2(1/\alpha)\rfloor$, $\bv$ is a depth-$d$ $[-1, 1]$-valued tree, and also $f\in \calF$. For any $\bepsilon\in \{-1, 1\}^d$, we have 
	$$\epsilon_t\cdot (f^\bepsilon(x_t(\bepsilon)) - v_t(\bepsilon)) = \epsilon_t\cdot \left(\sum_{i=t}^d \epsilon_i\cdot 2^{d-i}\right)\cdot \alpha = \left(2^{d-t}  - \sum_{i=t+1}^d \epsilon_i\epsilon_t\cdot 2^{d-i}\right)\cdot \alpha\ge \alpha > \frac{\alpha}{2}.$$
	Hence tree $\bx$ is shattered by $\calF$, which implies that $\fat(\calF, \alpha)\ge \lfloor \log_2(1/\alpha)\rfloor$.
\end{proof}

\subsection{Missing Proofs in \pref{sec: seq-offset-rademacher}}
\begin{proof}[Proof of \pref{thm: offset-lower-bound}]
	For fixed positive integer $n$, we let 
	$$\alpha_n = \frac{1}{2}\cdot \argmin_{\alpha > 0}\left\{\dseq(\calF, \alpha, \alpha/20)\cdot \frac{16}{\alpha^2}\le n\right\}.$$
	Since $\dseq(\calF, \alpha, \alpha/20) = \tOmega\left(\alpha^{-p}\right)$ for every $\alpha\ge 0$, we have 
	$$\dseq(\calF, \alpha_n, \alpha_n/20)\wedge \frac{n\alpha_n^2}{4} = \tOmega\left(n^{\frac{p}{p+2}}\right),\qquad \forall n\in \mathbb{Z}_+.$$
	In the following, we will prove that for any positive integer $n$, we have 
	$$\sup_{\bmu, \bx} \EE[\calR_n(\calF, \bmu, \bx, \bepsilon)] = \Omega\left(\dseq(\calF, \alpha_n, \alpha_n/20)\right).$$

	We fix $n$, and let $\alpha = \alpha_n$, $d = \dseq(\calF, \alpha_n, \alpha_n/20)\wedge (n\alpha_n^2/4)$. We let $\btx$ be the depth-$d$ $\calX$-valued tree shattered by $\calF$. Then according to \pref{def: seq-real-dim}, there exists a depth-$d$ $[-1, 1]\times [-1, 1]$-valued tree $\bs$ such that for any path $\btepsilon = (\tepsilon_{1:d})\in \{-1, 1\}^d$, there exists $f^{\btepsilon}\in \calF$ such that
	\begin{equation}\label{eq: eq-natarajan}
		\left|f^{\btepsilon}(\tx_t(\btepsilon)) - s_t(\btepsilon)[\tepsilon_t]\right| \leq \frac{\alpha}{20}\quad \text{and}\quad |s_t(\btepsilon)[1] - s_t(\btepsilon)[-1]| \ge \alpha\qquad \forall t\in [d],
	\end{equation}
	Without loss of generality, we assume $s_t(\btepsilon)[1]\ge s_t(\btepsilon)[-1]$. We define depth-$d$ $[-1, 1]$-valued $\btmu$ as 
	$$\tmu_t(\btepsilon) = \frac{s_t(\btepsilon)[-1] + s_t(\btepsilon)[1]}{2},\quad \forall \bepsilon\in \{-1, 1\}^d.$$

	In the following, we will construct trees $\bx$ and $\bmu$ such that $\EE[\calR_n(\calF, \bmu, \bx, \bepsilon)]\ge d/50$. 
	For a fixed path $\bepsilon\in \{-1, 1\}^n$, we first define an auxiliary path $\btepsilon = (\tepsilon_{1:d})\in \{-1, 1\}^d$ of length $d$, as well as $d$ integers $k_1, k_2, \ldots, k_d$ in the following way: calculate $\tepsilon_{1:d}$ and $k_{1:d}$ iteratively as
	$$k_t = \left\lfloor \frac{1}{(s_t(\btepsilon)[1] - s_t(\btepsilon)[-1])^2}\right\rfloor \vee 1,\qquad \forall t\in [d],$$
	and 
	$$\tepsilon_t = 2\mathbb{I}\bigg\{\sum_{j=1}^{k_t} \epsilon_{k_1 + \ldots + k_{t-1} + j} \ge 0\bigg\}-1,\qquad \forall t\in [d].$$
	Notice that according to the above definition, $k_t$ only depends on $\epsilon_{1: k_1 + \ldots + k_{t-1}}$, and $\tepsilon_t$ depends on $\epsilon_{1: k_1 + \ldots + k_t}$. Additionally, since $|s_t(\btepsilon)[1] - s_t(\btepsilon)[-1]|\ge \alpha$, we have 
	$$k_t\le \frac{1}{\alpha^2}\vee 1\le \frac{4}{\alpha^2},$$
	which implies $k_1 + \ldots + k_d\le n$ always holds due to the definition of $\alpha=\alpha_n$. Hence $k_{1:d}$ and $\tepsilon_{1:d}$ are all well-defined. We pad the tree $\bx$ with an arbitrary value $x_0$, resulting in the following block structure of 
     $(x_1(\bepsilon), x_2(\bepsilon), \ldots, x_n(\bepsilon))$ : 
	$$(\underbrace{\tx_1(\btepsilon), \tx_1(\btepsilon), \ldots, \tx_1(\btepsilon)}_{k_1 \text{ terms}}, \underbrace{\tx_2(\btepsilon), \tx_2(\btepsilon), \ldots, \tx_2(\btepsilon)}_{k_2\text{ terms}}, \ldots, \underbrace{\tx_d(\btepsilon), \tx_d(\btepsilon), \ldots, \tx_d(\btepsilon)}_{k_d\text{ terms}}, \underbrace{x_0, x_0, \ldots, x_0}_{(n - k_1 - k_2 - \ldots - k_d)\text{ terms}}),$$
	Similarly, the values $(\mu_1(\bepsilon), \mu_2(\bepsilon), \ldots, \mu_n(\bepsilon))$ 
	are of the form
	$$(\underbrace{\tmu_1(\btepsilon), \tmu_1(\btepsilon), \ldots, \tmu_1(\btepsilon)}_{k_1 \text{ terms}}, \underbrace{\tmu_2(\btepsilon), \tmu_2(\btepsilon), \ldots, \tmu_2(\btepsilon)}_{k_2 \text{ terms}}, \ldots, \underbrace{\tmu_d(\btepsilon), \tmu_d(\btepsilon), \ldots, \tmu_d(\btepsilon)}_{k_d \text{ terms}}, \underbrace{f^{\btepsilon}(x_0), f^{\btepsilon}(x_0), \ldots, f^{\btepsilon}(x_0)}_{(n - k_1 - k_2 - \ldots - k_d)\text{ terms}})$$

    We note that the construction of trees \(x_t(\bepsilon)\) and \(\mu_t(\bepsilon)\) here differs slightly from the construction used for the non-sequential setting in the proof of \pref{thm: offset-lower-bound-2}. In the sequential case, \( \mu \) is structured as a tree and can adapt based on the history of \( \epsilon \), allowing us to choose \( \mu \) as \( f^\bepsilon(x_0) \). 
    In contrast, in the non-sequential case, the choice of \( \mu \) must be independent of the history, so we are required to select global (i.e., fixed) values for \( \mu \). For this reason, \pref{thm: offset-lower-bound} only needs $d(\calF, \alpha, \alpha / 20) = \tOmega(\alpha^{-p})$, while \pref{thm: offset-lower-bound-2}, with the current proof, requires that $d(\calF, \alpha, \alpha / (20nC)) = \tOmega(\alpha^{-p})$.

	Next we write $\calR_n(\calF, \bmu, \bx, \bepsilon)$ in terms of segments $1$ to $d$. Noticing that $f^{\btepsilon}\in \calF$ for any depth-$d$ path $\btepsilon\in \{-1, 1\}^d$, if we use $[t, j]$ to denote $k_1 + \ldots + k_t + j$, we obtain
	\begin{align*} 
		&\hspace{-0.5cm} \calR_n(\calF, \bmu, \bx, \bepsilon)\\
		& = \sup_{f\in \calF} \Bigg\{\sum_{t=1}^d\sum_{j=1}^{k_t} C\cdot  \epsilon_{[t-1, j]} (f(x_{[t-1, j]}(\bepsilon)) - \mu_{[t-1, j]}(\bepsilon)) - (f(x_{[t-1, j]}(\bepsilon)) - \mu_{[t-1, j]}(\bepsilon))^2\\
		&\qquad + \sum_{j=1}^{n - k_1 - \ldots - k_d} C\cdot \epsilon_{[d, j]} (f(x_{[t-1, j]}(\bepsilon)) - \mu_{[d, j]}(\bepsilon)) - (f(x_{[d, j]}(\bepsilon)) - \mu_{[d, j]}(\bepsilon))^2\Bigg\}\\
		& \ge \sum_{t=1}^d\sum_{j=1}^{k_t} C\cdot\epsilon_{[t-1, j]} (f^{\btepsilon}(\tx_t(\bepsilon)) - \tmu_t(\btepsilon)) - (f^{\btepsilon}(\tx_t(\bepsilon)) - \tmu_t(\btepsilon))^2,
	\end{align*}
 	where in the last step we chose $f = f^{\btepsilon}\in \calF$. Next, for fixed $\bepsilon$, we define the random variable
	$$\hepsilon_t = \frac{1}{k_t} \sum_{j=1}^{k_t} \epsilon_{[t-1, j]}.$$
	Since $\tx_t(\btepsilon)$ and $\mu_t(\btepsilon)$ 
	are independent of $\epsilon_{[t-1, 1]}, \ldots, \epsilon_{[t-1, k_t]}$, we can write
	\begin{align*}
		\calR_n(\calF, \bmu, \bx, \bepsilon) & \ge \sum_{t=1}^d k_t\cdot \left(C\cdot\hepsilon_t(f^{\btepsilon}(\tx_t(\bepsilon)) - \tmu_t(\btepsilon)) - (f^{\btepsilon}(\tx_t(\bepsilon)) - \tmu_t(\btepsilon))^2\right). \numberthis \label{eq: r-n-segment}
	\end{align*}
	According to \pref{eq: eq-natarajan}, we have for any $\btepsilon$, 
	$$|f^{\btepsilon}(\tx_t(\btepsilon)) - s_t(\btepsilon)[\tepsilon_t]|\le \frac{\alpha}{20}\quad\text{and}\quad s_t(\btepsilon)[\tepsilon_t] - \mu_t(\btepsilon) = \sign(\hepsilon_t)\cdot \frac{s_t(\btepsilon)[1] - s_t(\btepsilon)[-1]}{2}.$$
	\pref{eq: eq-natarajan} also gives that $s_t(\btepsilon)[1] - s_t(\btepsilon)[-1] \ge \alpha$, which implies
	$$\frac{9}{10}\cdot \frac{s_t(\btepsilon)[1] - s_t(\btepsilon)[-1]}{2}\le \sign(\hepsilon_t)(f^{\btepsilon}(\tx_t(\btepsilon)) - \tmu_t(\btepsilon))\le \frac{11}{10}\cdot \frac{s_t(\btepsilon)[1] - s_t(\btepsilon)[-1]}{2}.$$
	Therefore we can lower bound \pref{eq: r-n-segment} by 
	$$\calR_n(\calF, \bmu, \bx, \bepsilon) \ge \sum_{t=1}^d k_t\cdot \left(C |\hepsilon_t|\cdot \frac{9}{10}\cdot \frac{s_t(\btepsilon)[1] - s_t(\btepsilon)[-1]}{2} - \frac{121}{100}\cdot\left(\frac{s_t(\btepsilon)[1] - s_t(\btepsilon)[-1]}{2}\right)^2\right).$$ 
	We notice that in the $t$-th summand in the right hand side, the only term which depend on $\epsilon_{[t-1, 1]: [t-1, k_t]}$ is $|\hepsilon_t|$, and all other terms only depends on $\epsilon_{1: [t-1: 0]}$. Hence we can lower bound $\EE\left[\calR_n(\calF, \bmu, \bx, \bepsilon)\right]$ by
	\begin{align*} 
		&\hspace{-0.5cm} \EE\left[\calR_n(\calF, \bmu, \bx, \bepsilon)\right]\\
		& \ge \sum_{t=1}^d \EE\left[k_t\cdot \left(C |\hepsilon_t|\cdot \frac{9}{10}\cdot \frac{s_t(\btepsilon)[1] - s_t(\btepsilon)[-1]}{2} - \frac{121}{100}\cdot\left(\frac{s_t(\btepsilon)[1] - s_t(\btepsilon)[-1]}{2}\right)^2\right)\right]\\
		& = \sum_{t=1}^d \EE\left[k_t\cdot \left(C \EE\left[|\hepsilon_t|\mid t_{1:[t-1, 0]}\right]\cdot \frac{9}{10}\cdot \frac{s_t(\btepsilon)[1] - s_t(\btepsilon)[-1]}{2} - \frac{121}{100}\cdot\left(\frac{s_t(\btepsilon)[1] - s_t(\btepsilon)[-1]}{2}\right)^2\right)\right]\\
		& \ge \sum_{t=1}^d \EE\left[k_t\cdot \left(C \sqrt{\frac{1}{2k_t}}\cdot \frac{9}{10}\cdot \frac{s_t(\btepsilon)[1] - s_t(\btepsilon)[-1]}{2} - \frac{121}{100}\cdot\left(\frac{s_t(\btepsilon)[1] - s_t(\btepsilon)[-1]}{2}\right)^2\right)\right]\\ \numberthis \label{eq: equation-lower-decompose}
	\end{align*}
	where the last inequality uses the Khintchine inequality \cite{haagerup1981best}. Next notice that according to our choice of $k_t$, 
	$$k_t = \left\lfloor \frac{1}{(s_t(\btepsilon)[1] - s_t(\btepsilon)[-1])^2}\right\rfloor \vee 1 \in \left[\frac{1}{2(s_t(\btepsilon)[1] - s_t(\btepsilon)[-1])^2}, \frac{4}{(s_t(\btepsilon)[1] - s_t(\btepsilon)[-1])^2}\right], $$
	which implies that $1/\sqrt{2}\le \sqrt{k_t} (s_t(\btepsilon)[1] - s_t(\btepsilon)[-1])\le 2$. Therefore, since $C\ge 2$, we have
	\begin{align*}
		\text{RHS of }\pref{eq: equation-lower-decompose} & \ge \sum_{t=1}^d \EE\left[k_t\cdot \left(C \sqrt{\frac{1}{2k_t}}\cdot \frac{9}{10}\cdot \frac{s_t(\btepsilon)[1] - s_t(\btepsilon)[-1]}{2} - \frac{121}{100}\cdot\left(\frac{s_t(\btepsilon)[1] - s_t(\btepsilon)[-1]}{2}\right)^2\right)\right]\\
		& \ge \sum_{t=1}^d \EE\left[\left(\frac{9C}{20\sqrt{2}} - \frac{121}{200}\right)\cdot \sqrt{k_t}\cdot |s_t(\btepsilon)[1] - s_t(\btepsilon)[-1]|\right] \ge \frac{d}{50}.
	\end{align*}
	Therefore, we obtain that 
	$$\sup_{\bmu, \bx} \EE[\calR_n(\calF, \bmu, \bx, \bepsilon)]\ge \frac{d}{50} = \tOmega\left(n^{\frac{p}{p+2}}\right).$$
\end{proof}

\begin{proof}[Proof of \pref{corr: sequential}]
    Since $\sup_{\bx}\log \calNseq_\infty(\calF, \bx, \alpha) = \tOmega\left(\alpha^{-p}\right)$, according to \pref{prop: seq-covering-fat} we have 
    $$\dseq(\calF, \alpha, \alpha/20) = \tOmega\left(\alpha^{-p}\right).$$
    Next, we call \pref{thm: offset-lower-bound} with $C = 2$. According to \cite[Lemma 4]{rakhlin2014online}, we obtain
    $$\calVseq_n(\calF) = \tOmega\left(n^{\frac{p}{p+2}}\right).$$
\end{proof}

\end{document}